\documentclass[twoside]{article}

\usepackage[accepted]{aistats2020}

\usepackage[utf8]{inputenc} %
\usepackage[T1]{fontenc}    %
\usepackage[hidelinks]{hyperref}       %
\usepackage{url}            %
\usepackage{booktabs}       %
\usepackage{amsfonts}       %
\usepackage{nicefrac}       %
\usepackage{microtype}      %
\usepackage{algorithmicx,algpseudocode,algorithm}

\usepackage{comment}
\usepackage{enumitem} 
\usepackage{bm}
\usepackage{placeins}

\usepackage{wrapfig,tikz}
\usepackage{amsmath,longtable,fancyhdr,booktabs,multirow,graphicx,float}
\usepackage{adjustbox, bigstrut, tabularx, multirow, makecell, diagbox}
\usepackage{amssymb,xcolor,amsthm}
\usepackage{subcaption}
\usepackage{color}
\usepackage{colortbl}
\usepackage{theoremref}
\definecolor{smcolor}{rgb}{0.5490196078431373, 0.33725490196078434, 0.29411764705882354}
\definecolor{cpcolor}{rgb}{0.8392156862745098, 0.15294117647058825, 0.1568627450980392}
\definecolor{dsmcolor}{rgb}{0.17254901960784313, 0.6274509803921569, 0.17254901960784313}
\definecolor{ssmcolor}{rgb}{1.0, 0.4980392156862745, 0.054901960784313725}
\definecolor{ssmvrcolor}{rgb}{0.12156862745098039, 0.4666666666666667, 0.7058823529411765}
\definecolor{bpcolor}{rgb}{0.5803921568627451, 0.403921568627451, 0.7411764705882353}
\definecolor{mlecolor}{rgb}{0.8901960784313725, 0.4666666666666667, 0.7607843137254902}
\definecolor{elbocolor}{rgb}{0.4980392156862745, 0.4980392156862745, 0.4980392156862745}
\definecolor{steincolor}{rgb}{0.7372549019607844, 0.7411764705882353, 0.13333333333333333}
\definecolor{spectralcolor}{rgb}{0.09019607843137255, 0.7450980392156863, 0.8117647058823529}

\newcommand{\mbf}[1]{\mathbf{#1}}

\newcommand{\mbb}[1]{\mathbb{#1}}
\newcommand{\ud}{\mathrm{d}}
\newcommand{\up}{\mathrm}

\newcommand{\mcal}{\mathcal}
\newcommand{\pdata}{p_{\mathrm{data}}}
\newcommand{\norm}[1]{\left\lVert#1\right\rVert}

\newtheorem{lemma}{Lemma}
\newtheorem{theorem}{Theorem}
\newtheorem{corollary}{Corollary}

\newtheorem{innercustomthm}{Theorem}
\newenvironment{customthm}[1]
{\renewcommand\theinnercustomthm{#1}\innercustomthm}
{\endinnercustomthm}

\newcommand{\be}{\begin{equation}}
	\newcommand{\ee}{\end{equation}}

\definecolor{Gray}{gray}{0.85}
\definecolor{LightCyan}{rgb}{0.88,1,1}

\newcolumntype{a}{>{\columncolor{Gray}}c}
\newcolumntype{b}{>{\columncolor{white}}c}

\DeclareMathOperator*{\arginf}{arg\,inf}

\makeatletter
\usepackage{xspace}
\def\@onedot{\ifx\@let@token.\else.\null\fi\xspace}
\DeclareRobustCommand\onedot{\futurelet\@let@token\@onedot}

\newcommand{\figref}[1]{Fig\onedot~\ref{#1}}

\newcommand{\secref}[1]{Section~\ref{#1}}
\newcommand{\tabref}[1]{Tab\onedot~\ref{#1}}

\newcommand{\eqnref}[1]{Eq\onedot~\eqref{#1}}

\newcommand{\appref}[1]{Appendix~\ref{#1}}

\newcommand{\corref}[1]{Corollary~\ref{#1}}
\newcommand{\lemref}[1]{Lemma~\ref{#1}}

\newcommand{\bfx}{\mathbf{x}}
\newcommand{\bfX}{\mathbf{X}}
\newcommand{\bfv}{\mathbf{v}}
\newcommand{\bfz}{\mathbf{z}}

\newcommand{\bfT}{\mathbf{T}}
\newcommand{\bftheta}{{\boldsymbol{\theta}}}
\newcommand{\bfalpha}{{\boldsymbol{\alpha}}}

\def\eg{\emph{e.g}\onedot}

\def\ie{\emph{i.e}\onedot}

\def\iid{i.i.d\onedot}

\newcommand{\KLD}[2]{D_{\mathrm{KL}}\left(\left. #1 \ \middle\Vert \ #2 \right.\right)}
\usepackage[round]{natbib}

\begin{document}

\twocolumn[

\aistatstitle{Gaussianization Flows}

\aistatsauthor{ Chenlin Meng* \And Yang Song* \And Jiaming Song \And  Stefano Ermon }

\aistatsaddress{Computer Science Department, Stanford University}]

\begin{abstract}
Iterative Gaussianization is a fixed-point iteration procedure that can transform any continuous random vector into a Gaussian one. Based on iterative Gaussianization, we propose a new type of normalizing flow model that enables both efficient computation of likelihoods and efficient inversion for sample generation. We demonstrate that these models, named \textit{Gaussianization flows}, are universal approximators for continuous probability distributions under some regularity conditions. Because of this guaranteed expressivity, they can capture multimodal target distributions without compromising the efficiency of sample generation. Experimentally, we show that Gaussianization flows achieve better or comparable performance on several tabular datasets compared to other efficiently invertible flow models such as Real NVP, Glow and FFJORD. In particular, Gaussianization flows are easier to initialize, demonstrate better robustness with respect to different transformations of the training data, and generalize better on small training sets.
\end{abstract}
\section{INTRODUCTION}
Maximum likelihood is a widely adopted approach for 
density estimation. However, for very expressive probabilistic models, \eg, those parameterized by deep neural networks, evaluating likelihood can be intractable. Several special architectures have been proposed to build probabilistic models with tractable likelihoods. One such family of models is \emph{normalizing flows}~\citep{rezende15variational,dinh2014nice,dinh2016density}. These models learn a bijective mapping $\bfT$ that pushes forward the data distribution to a simple target distribution (typically Gaussian or uniform) such that the log determinant of the transformation's Jacobian ($\log |\det J_\bfT|$) is efficient to compute. The corresponding likelihood can then be efficiently computed via the change of variables formula, enabling efficient training via maximum likelihood.

Given a density model, it is often desirable to generate samples from it in an efficient way. This requires an additional property for normalizing flow models: the inverse of $\bfT$ must also be easy to compute. Unfortunately, even though flow models are invertible by construction, they are not always efficiently invertible in practice. For example, models like MAF~\citep{maf}, NAF~\citep{huang2018neural}, Block-NAF~\citep{de2019block} all need $D$ times more computation for inversion than for likelihood evaluation, where $D$ is the data dimension. Continuous flow models, such as Neural ODE~\citep{chen2018neural} and FFJORD~\citep{FFJORD}, take roughly the same time for inversion and likelihood evaluation, but both directions involve slow numerical integration procedures. Models based on coupling layers, \eg, Real NVP~\citep{nvp} and Glow~\citep{glow}, have efficient procedures for both inversion and likelihood computation, yet it is unclear whether their architectures are sufficiently expressive to capture all distributions.

To explore different flow architectures that are expressive and permit efficient sampling, we draw inspiration from iterative Gaussianization. First proposed in \citet{chen2001gaussianization}, it is an iterative approach to transform the data distribution to a standard (multivariate) Gaussian distribution. Specifically, we first transform each data point with a linear mapping (typically an orthogonal matrix computed by ICA or PCA), and then individually ``Gaussianize'' the marginal distributions of each data dimension. This is achieved by estimating each univariate CDF, mapping each data dimension to a uniform random variable, and then transforming it to a Gaussian by CDF inversion. 
Intuitively, the linear mapping in Gaussianization amounts to finding a specific direction where the marginals of the data distribution are as ``non-Gaussian'' as possible; this ``non-Gaussianity'' is reduced by the subsequent Gaussianization step performed for each marginal distribution. As proved in \citet{chen2001gaussianization}, the transformed data distribution converges to a  standard normal if this procedure is repeated a sufficiently large number of times (under some conditions). Though theoretically satisfying, this method has many limitations in practice. First, Gaussianizing marginal distributions is practically difficult, even in the univariate case, because   non-parametric methods for CDF estimation (such as kernel density estimation) can be inaccurate and hard to tune. Second, finding optimal linear mappings such that the marginal distributions are ``non-Gaussian'' is challenging and traditional methods such as linear ICA do not have closed-form solutions and can be very slow to run for large scale datasets.

To mitigate these limitations while preserving theoretical guarantees, we propose to parameterize the Gaussianization procedure to make it jointly trainable, in lieu of following the original iterative refining approach. This results in a new family of flow models named \emph{Gaussianization flows}. More specifically, we parameterize the linear mapping by stacking several Householder transformations with learnable parameters. After this linear mapping, we parameterize an element-wise non-linear transformation by composing the inverse Gaussian CDF with the CDF of a trainable mixture of logistic distributions. Combining the linear mapping and element-wise non-linear transformation, we get a differentiable \emph{Gaussianization module} whose Jacobian determinant is available in closed-form, and inversion is easy to compute. We can stack several Gaussianization modules to form a Gaussianization flow model which is also easy to invert.

We can show that Gaussianization flows are universal approximators when the model is sufficiently wide and deep, meaning that the model architecture is theoretically expressive enough to transform any data distribution with strictly positive density to a Gaussian distribution (under some regularity conditions). Due to the connection between Gaussianization flows and iterative Gaussianization, the layers of Gaussianization flows have a natural interpretation. For example, the mixture of logistics in a Gaussianization flow should ideally capture the marginal distribution obtained after applying the Householder layer. We can therefore initialize the parameters of the mixture of logistic used for Gaussianization using a kernel density estimator with logistic kernels for better training.
Because of the non-parametric nature of kernel density estimation, this intialization is more adaptive, providing some robustness with respect to re-parameterizations of the data. 

In our experiments, we demonstrate that Gaussianization flows achieve better or comparable performance on density estimation for tabular data, compared to some efficient invertible baselines such as Real NVP, Glow and FFJORD. In particular, we achieve better performance when the number of training data points is limited, and our models show more robustness to reparameterizations of the data.

\section{BACKGROUND}
\subsection{Density Estimation with Flow Models}
Let $\mcal{D} = \{\bfx_j \in \mbb{R}^D\}_{j=1}^{N}$ %
be a dataset of continuous observations which are \iid samples from an unknown continuous data distribution (denoted as $\pdata$). Given this dataset $\mcal{D}$, the goal of density estimation is to approximate $\pdata$ with a probabilistic model parameterized by $\bftheta$ (denoted as $p_\bftheta$). Specifically, we learn an invertible model $\bfT_\bftheta: \mbb{R}^D \to \mbb{R}^D$, which performs a bijective, differentiable transformation of $\bfx$ to $\bfz = \bfT_\bftheta(\bfx)$. Using the change of variables formula, %
\begin{align*}
    p_\bftheta(\bfx) = p_z(\bfT_\bftheta(\bfx)) \left\vert \det \frac{\partial \bfT_\bftheta(\bfx)}{\partial \bfx} \right\vert = p_z(\bfz) 
    |\det J_{\bfT_\bftheta}(\bfx)|\label{eqn:cofv},
\end{align*}
where $\det J_{\bfT_\bftheta}(\bfx)$ denotes the determinant of the Jacobian matrix evaluated at $\bfx$, and $p_z(\bfz)$ is a simple fixed distribution with tractable density (\eg the multivariate standard Gaussian $\mcal{N}(\mathbf{0}, \mathbf{I})$). Note that in order to evaluate the likelihood $p_\bftheta(\bfx)$, the determinant of Jacobian $\det J_{\bfT_\bftheta}(\bfx)$ must be easy to compute. Models with this property are named \emph{normalizing flow models}~\citep{rezende15variational}.

Multiple flow models $\bfT_1, \bfT_2, \cdots, \bfT_L$ can be stacked together to yield a deeper and more expressive model $\bfT = \bfT_1 \circ \bfT_2 \circ \cdots \circ \bfT_L$. Since $\bfT^{-1} = \bfT^{-1}_L \circ \bfT^{-1}_{L-1} \circ \cdots \circ \bfT^{-1}_1$, and $\det J_\bfT = \det J_{\bfT_1} \det J_{\bfT_2} \cdots \det J_{\bfT_L}$, as long as each component $\bfT_i$ is invertible and has tractable determinant of Jacobian, the combined model $\bfT$ also shares such properties.

\subsection{Iterative Gaussianization}
\label{sec:2.2}
Training a flow model with maximum likelihood amounts to solving
\begin{multline}
    \min_\bftheta \mbb{E}_{p_\text{data}(\bfx)}[-\log p_\bftheta(\bfx)] \\
    = \min_\bftheta \KLD{p_\text{data}(\bfx)}{p_\bftheta(\bfx)} + \up{const}.
\end{multline}
When $p_\bftheta(\bfx)$ is the likelihood of a flow model $\bfT_\bftheta(\bfx)$ given by \eqnref{eqn:cofv}, we can transform the above objective using the fact that KL divergence is invariant to bijective mappings of random variables, which gives us
\begin{multline}
    \min_\bftheta \KLD{p_\text{data}(\bfx)}{p_\bftheta(\bfx)} + \up{const}\\
    = \min_\bftheta \KLD{p_{\bfT_\bftheta}(\bfz)}{\mcal{N}(\mbf{0}, \mbf{I})} + \up{const}\label{eqn:gauss_obj},
\end{multline}
where $p_{\bfT_\bftheta}$ denotes the distribution of $\bfz = \bfT_\bftheta(\bfx)$, when $\bfx$ is sampled from $p_\bftheta(\bfx)$. Intuitively, \eqnref{eqn:gauss_obj} means that training a flow model with maximum likelihood is equivalent to finding an invertible transformation to warp the data distribution to a multivariate standard normal distribution. This task is well-known as \emph{Gaussianization}~\citep{chen2001gaussianization}.%

For one-dimensional (univariate) data $x \sim \pdata(x)$, one could perform Gaussianzation by estimating its cumulative density function (CDF, \eg using kernel density estimation) and applying the inverse Gaussian CDF. To see this, let $\Phi$ be the CDF of the standard normal distribution, and $F_\text{data}$ be the CDF of the data distribution, we can transform any random variable $x \sim \pdata$ to a Gaussian random variable $z$ by $z = \Phi^{-1}\circ F_\text{data}(x)$.

For high dimensional data, one key observation is that the KL divergence between a distribution $p(\bfx)$ and a multivariate standard Gaussian distribution can be decomposed as follows~\citep{chen2001gaussianization}:
\begin{equation}
    \KLD{p(\bfx)}{\mcal{N}(\mbf{0}, \mbf{I})}\triangleq J(\bfx) = I(\bfx)+J_m(\bfx) \label{eq:kl_gauss}
\end{equation}
where $I(\bfx)$ is the multi-information that measures the statistical dependence among components of $\bfx$:
\begin{equation}
I(\bfx) = \KLD{p(\bfx)}{\prod_i^{D} p_i(x^{(i)})}, \label{eq:ix}
\end{equation}
and $J_m(\bfx)$ is the sum of KL divergences between the marginal distributions and univariate standard normal distributions:
\begin{equation}
J_m(\bfx)=\sum_{i=1}^{D}\KLD{p_i(x^{(i)})}{\mcal{N}(0,1)}. \label{eq:jm}
\end{equation}
Here we represent $\bfx=(x^{(1)},x^{(2)},\cdots,x^{(D)})^\intercal$, and let $p_i(x^{(i)})$ be the marginal distribution of $p(\bfx)$.
Intuitively, to transform the data distribution into a multivariate unit Gaussian, we need to make each dimension independent ($I(\bfx) = 0$), and each marginal distribution univariate standard normal ($J_m(\bfx) = 0$).

Based on the decomposition \eqnref{eq:kl_gauss}, a particular iterative Gaussianization~\citep{chen2001gaussianization} approach---Rotation-Based Iterative Gaussianization (RBIG,~\citet{laparra2011iterative})---alternates between applying one-dimensional Gaussianization and rotations to the data. Specifically, RBIG estimates the marginal distribution corresponding to each dimension of the data distribution, and performs one-dimensional Gaussianization of all marginal distributions. Then, RBIG applies a rotation matrix to the transformed data.

The rationale behind RBIG is that dimension-wise Gaussianization will decrease $J_m(\bfx)$ and leave $I(\bfx)$ invariant, due to the fact $I(\bfx)$ is invariant under dimension-wise invertible transformations~\citep{laparra2011iterative}, whereas applying rotation to $p(\bfx)$ will not modify the overall KL divergence objective $I(\bfx) + J_m(\bfx)$ since KL is invariant under bijective transformations (rotation in particular) and $\mcal{N}(\mbf{0},\mbf{I})$ is rotationally invariant. Therefore, $\KLD{p(\bfx)}{\mcal{N}(\mbf{0},\mbf{I})}$ will not increase (typically decreases) at each RBIG iteration. %
To improve the performance of RBIG, one could consider rotation operators that make $J_m(\bfx)$ as large as possible, so that the subsequent marginal Gaussianization step removes $J_m(\bfx)$ and results in a large decrease in $\KLD{p(\bfx)}{\mcal{N}(\mbf{0},\mbf{I})}$. Popular choices of rotation matrices include random matrices and those computed by independent component analysis (ICA) and principal component analysis (PCA). However, all three candidates are less than desirable. For random rotations and PCA, 
the procedure 
could require many RBIG steps to converge~\citep{laparra2011iterative}. ICA, on the other hand, is optimal yet does not have closed-form solutions and is expensive to compute in practice.

\section{METHOD}
While iterative Gaussianization possesses the ability to transform a complex distribution to standard normal, density estimation with iterative Gaussianization is still difficult, because of the following challenges:
\begin{itemize}
    \item One-dimensional (1D) Gaussianization is challenging for certain data distributions;
    \item Finding optimal rotation matrices is challenging (as in the case of ICA rotation matrices, which have no closed form solution).
\end{itemize}
In this section, we address these challenges with a new type of invertible flow model based on the iterative Gaussianization (RBIG) method, named \textit{Gaussianization Flows} (GF). Specifically, GF improves the two components of RBIG where we replace 1D Gaussianization with a \textit{trainable kernel layer} and a fixed rotation matrix with a \textit{trainable orthogonal matrix layer}.

\subsection{Building Trainable Kernel Layers}
\label{sec:trainable_kernels}

Marginal Gaussianization plays a crucial role in RBIG since it reduces the objective value $J_m(\bfx)$ in \eqnref{eq:jm} and is the only procedure that decreases the KL objective in \eqnref{eq:kl_gauss} (rotation does not change the KL divergence because KL is invariant to bijective mappings, however, it enables progress in the next iteration). For a set of 1D scalars $\{x_j\}_{j=1}^{M}$, one could perform Gaussianization by first estimating a CDF (denoted as $F_{\text{data}}(x)$), and then applying the transformation $\phi: x \mapsto \Phi^{-1} \circ F_{\text{data}}(x)$ where $\Phi$ is the CDF for a 1D standard Gaussian. %

One approach to estimate the CDF is via 1D density estimation, where the CDF can be computed from the PDF by taking the integral. As we are assuming the underlying data distribution is continuous, we can naturally employ kernel density estimation (KDE) methods to fit the data PDF, and then obtain the CDF by integrating out the kernels in closed-form. However, there are two shortcomings of KDE for large-scale density estimation. Firstly, the complexity of computing the KDE for each sample scales quadratically with the number of samples, making it prohibitive for larger batches/datasets; secondly, the performance of KDE largely depends on the sample size~\citep{parzen1962estimation,devroye1979} and bandwidth selection \citep{sheather2004density}, yet optimal bandwidths are difficult to obtain even with good bandwidth selection heuristics \citep{scott1985kernel}.

To alleviate the limitations of existing non-parametric KDE approaches, we propose to learn a ``parameterized KDE'' for each data dimension, leading to \textit{trainable kernel layers}. For each data dimension (indexed by $d=1,2,\cdots,D$), we learn a set of anchor points $\{\mu_j^{(d)}\}_{j=1}^{K}$ and bandwidth parameters $\{h^{(d)}_j\}_{j=1}^K$.
This leads to a total of $2KD$ parameters for a trainable kernel layer. Mathematically, we parameterize a CDF with the following
\begin{align} 
  F_{\bftheta}^{(d)}(x) \triangleq \frac{1}{K} \sum_{j=1}^{K} \sigma\left(\frac{x^{(d)} - \mu_j^{(d)}}{h_j^{(d)}}\right), d=1,\cdots,D,
\end{align}
where $\sigma(\cdot)$ denotes the sigmoid function throughout the paper, and $\bftheta$ denotes the collection of all trainable parameters ($\{\mu_j^{(d)}\}_{j=1}^{K}$ and $\{h^{(d)}_j\}_{j=1}^K$). Learning this CDF amounts to performing KDE with a logistic kernel when $\sigma(\cdot)$ is the sigmoid function. Then, the Gaussianization procedure for dimension $d$ can be parameterized as 
\begin{align}
    \Psi^{(d)}_\bftheta(x) \triangleq \Phi^{-1} \circ F_{\bftheta}^{(d)}(x), \quad d=1,\cdots,D,
\end{align}
and we denote $\Psi_\bftheta = (\Psi_\bftheta^{(1)}, \Psi_\bftheta^{(2)}, \cdots, \Psi_\bftheta^{(D)})^\intercal$.

By making anchor points and bandwidths trainable, our parametric trainable kernel layer can be more sample efficient compared to the traditional non-parametric KDE approach (when trained, for example, with maximum likelihood). We find that 20 to 100 anchor points work well in practice. In stark contrast, na\"{i}ve KDE needs thousands of sample points to get comparable results, which is particularly inefficient given that the computational complexity scales quadratically with respect to $K$. 

We note that $\Psi$ is a transformation with a Jacobian whose determinant is tractable. Additionally, $\Psi$ can be efficiently inverted:
\begin{itemize}
    \item $\Phi, \Phi', \Phi^{-1}$ are not computable by elementary functions, yet they can be efficiently evaluated via numerical methods.
    \item As both $\Phi^{-1}$ and $F_{\bftheta}^{(d)}$ are monotonic, $\Psi_\bftheta^{(d)} = \Phi^{-1} \circ F_{\bftheta}^{(d)}$ is also monotonic. We can therefore efficiently invert $\Psi_\bftheta$ by inverting all of its dimensions with the \emph{bisection method in parallel}, as $\Psi_\bftheta$ is element-wise. %
    \item The Jacobian of $\Psi$ is a diagonal matrix. The log-determinant is therefore the sum of the log-derivatives of $\Phi^{-1} \circ F_{\bftheta}^{(d)}(x)$ over all dimensions.
\end{itemize}

\subsection{Building Trainable Rotation Matrix Layers}
\label{sec:rotation_matrices}
In iterative Gaussianization, we transform the data using a rotation matrix after the marginal Gaussianization step. As mentioned in \secref{sec:2.2}, finding a good rotation matrix is challenging using methods like ICA or PCA. Here, we discuss our approach to finding rotations by optimizing trainable rotation matrices.
\subsubsection{Householder Reflections}
We can parameterize the rotation matrix using Householder reflections, defined for any vector $\bfv \in \mathbb{R}^D$: %
\begin{align}
    H = I - \frac{2 \bfv \bfv^\intercal}{\norm{\bfv}_2^2}.
\end{align}
Any $D\times D$ orthogonal matrix $R$ can be represented as the product of at most $D$ Householder reflections~\citep{tomczak2016improving}, \ie, $R = H_1 H_2 \cdots H_D$. 

By parameterizing the rotation matrix with multiple trainable Householder reflections, we define a \textit{trainable orthogonal matrix layer}. Since the inverse of a rotation matrix is the transpose of itself, one can efficiently obtain the inverse by multiplying the transpose of the orthogonal matrix. Moreover, because the Jacobian determinant of an orthogonal transformation is always one, we can easily compute the Jacobian determinant of this layer, which is also equal to one.

One caveat is that each Householder reflection requires $D$ parameters, and thus fully parameterizing a rotation matrix will require $O(D^2)$ parameters. This is reasonable when the data dimension is small. However, this may no longer be feasible in cases where $D$ is large. %
For example, CIFAR-10~\citep{krizhevsky2009learning} images have $D = 3072$, and ImageNet~\citep{imagenet_cvpr09} images can have $D$ as large as $10^6$. %
In such cases, one may need to trade off model flexibility for computational efficiency by using a smaller number ($<D$) of Householder reflections. Below, we explore one such approach that exploits the structure of images and utilizes a patch-based parameterization of rotation matrices to significantly reduce the number of parameters.

\subsubsection{Patch-Based Rotation Matrices}

Intuitively, a pixel in an image is more correlated to its neighboring pixels than far away ones. Based on this intuition, we propose ``patch-based'' Householder reflections for parameterizing rotation matrices for images. Recalling that the role of the rotation matrix in RBIG is to render the components as independent as possible, patch-based Householder reflections are designed to focus on the components where we expect to get the biggest gains, i.e., the ones that are farthest from being independent.

For an image with dimension $L\times L$, the rotation matrix will have size $L^2\times L^2$. Assuming $p$ is a divisor of $L$ and $L=p\times k$, we can partition the matrix into $k^2\times k^2$ smaller blocks each with size $p^2\times p^2$. Instead of directly parameterizing the $L^2\times L^2$ rotation matrix using $L^2$ Householder reflections, we parameterize a block-diagonal rotation matrix with $k^2$ blocks. Each block on the diagonal is a $p^2\times p^2$ rotation matrix, which requires $p^2$ Householder reflections to parameterize. Since rotation is now only performed in each $p\times p$-dimensional subspace, we leverage a ``shift'' operation on the input vectors to introduce dependency across different rotational subspaces. %
We call this block-diagonal rotation matrix a ``patch-based rotation matrix'' (see \figref{fig:patch_rotation}), and relegate extra details to Appendix~\ref{app:patch}.

\begin{figure}%
    \centering
        \includegraphics[width=0.5\textwidth]{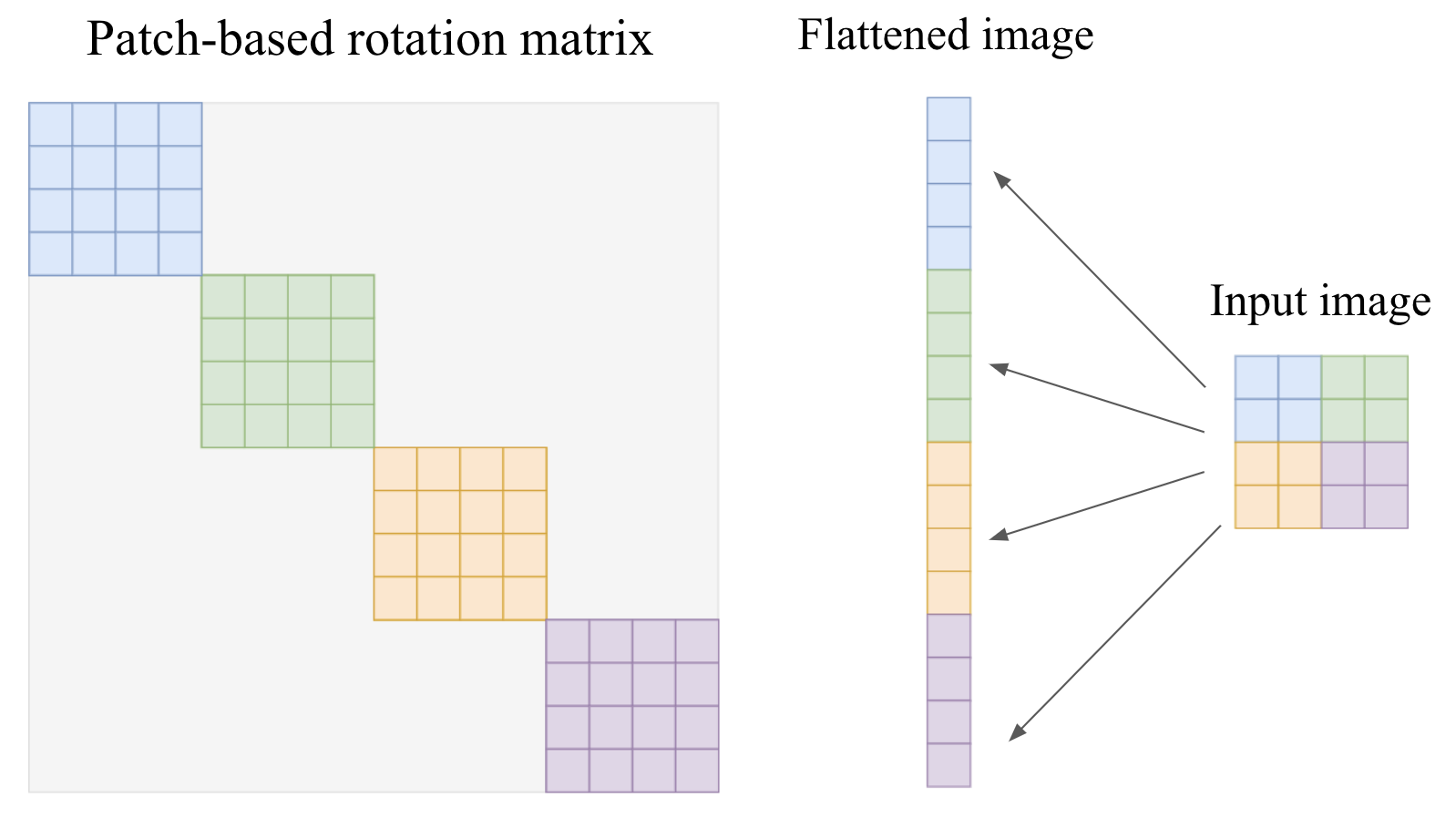}
    \caption{A patch-based rotation matrix where $L=4, p=2$ and $k=2$. %
    All entries with the grey color are zeros. Each $4\times 4$ block on the diagonal corresponds to a new subspace of neighboring pixels, where we perform Householder reflections.}
    \label{fig:patch_rotation}
\end{figure}

\subsection{Deep Gaussianization Flows}

Our proposed model, \textit{Gaussianization flow}, is constructed by stacking trainable kernel layers (Section~\ref{sec:trainable_kernels}) and orthogonal matrix layers (Section~\ref{sec:rotation_matrices}) alternatively. Formally, we define an Gaussianization flow with $L$ trainable kernel layers and orthogonal layers as:
\begin{align}
T_\bftheta(\bfx) = \Psi_{\bftheta_L}\circ R_L \circ \Psi_{\bftheta_{L-1}} \circ \cdots \circ \Psi_{\bftheta_1}\circ R_1 \bfx
\end{align}
where $\bftheta$ denotes the collection of all parameters. 

Note that both forward and backward computations of the Gaussianization flow are efficient, and the log determinant of its Jacobian can be computed in closed-form. Consequently, we can \emph{train} Gaussianization flows jointly with maximum likelihood, as well as producing samples efficiently. This is to the contrary of RBIG, which is a non-trainable iterative procedure.

\subsection{Gaussianization Flows are Universal Approximators}
We hereby prove that Gaussianization flows can transform any continuous distribution with a compact support to a standard normal, given that the number of layers and the number of parameters in each layer are sufficiently large. Ours is the first universal approximation result we are aware of for efficiently invertible normalizing flows.

Our results closely follow that of \cite{chen2001gaussianization}. However, we note that their results are weaker than what we need: they assume the marginal Gaussianization step can be done perfectly, whereas we use the learnable kernel layers for doing marginal Gaussianization. We defer all proofs to Appendix~\ref{app:proof}.

Our proof starts by showing that mixtures of logistic distributions (as used in our learnable kernel layers) are universal approximators for continuous densities (see \lemref{lem:universal} in Appendix). Therefore, our learnable kernel layers will be able to do arbitrarily good marginal Gaussianization when sufficiently many anchor points are used. Based on this, we show that Gaussianization flow is a universal approximator given a sufficient number of layers:
\begin{theorem}\label{thm:1}
Let $p$ be any continuous distribution supported on a compact set $\mcal{X} \subset \mbb{R}^D$, and $\inf_{x\in\mcal{X}} p(x) \geq \delta$ for some constant $\delta > 0$. Then, there exists a sequence of marginal Gaussianization layers $\{\Psi_{\bftheta_1}, \Psi_{\bftheta_2},\cdots, \Psi_{\bftheta_k}, \cdots\}$ and rotation matrices $\{R_1, R_2, \cdots, R_k, \cdots\}$ such that the transformed random variable
\begin{gather*}
    \Psi_{\bftheta_k}\circ R_k \circ \Psi_{\bftheta_{k-1}} \circ R_{k-1} \circ \cdots \circ \Psi_{\bftheta_1}\circ R_1\bfX
    \stackrel{d}{\to} \mcal{N}(\mbf{0}, \mbf{I}),
\end{gather*}
where $\bfX \sim p$.
\end{theorem}

\subsection{Building Invertible Networks with Proper Initializations}
Since our Gaussianization flow is a trainble extension of RBIG, we propose to provide good initializations for Gaussianization flows using RBIG. In the \textit{trainable rotation matrix layers}, we randomly initialize each Householder reflection vector with samples from an isotropic Gaussian. This amounts to using random rotation matrices in RBIG. We abstain from using ICA/PCA layers for providing initialization both for the aforementioned computational issues, and for the fact that they provide similar results in practice.

In the \textit{trainable kernel layer}, we consider a data-dependent initialization approach, %
using $N$ random samples from the dataset. To initialize KDE anchor points in the first layer, we randomly draw $N$ samples from the dataset. %
More generally, we initialize the KDE anchor points at layer $l+1$ using the outputs of the $l$-th trainable rotation matrix layer.

In fact, the initial state of our model corresponds to an iterative Gaussianization method, which, as shown previously, is capable of capturing distributions to a certain level. This allows our GF to outperform other normalizing flows at initial iterations. Because of the good initialization, our model also exhibit better robustness with respect to re-parameterizations of the data.

\section{EXPERIMENTS}\label{sec:exp}

We evaluate our Gaussianization Flow (GF) on several datasets; these include synthetic 2D toy datasets, benchmark tabular UCI datasets~\citep{maf} (Power, Gas, Hepmass, MiniBoone, BSDS300) and two image datasets (MNIST and Fashion-MNIST). We compare with several popular invertible models for density estimation, including RealNVP~\citep{dinh2016density}, Glow~\citep{glow}, FFJORD~\citep{FFJORD}, MAF~\citep{maf}, TAN~\citep{oliva2018transformation} and NAF~\citep{huang2018neural}; we also compare directly with RBIG~\citep{laparra2011iterative} for reference. %
\begin{figure}[H]%
    \centering
    \begin{subfigure}[b]{0.1\textwidth}
        \includegraphics[width=\textwidth]{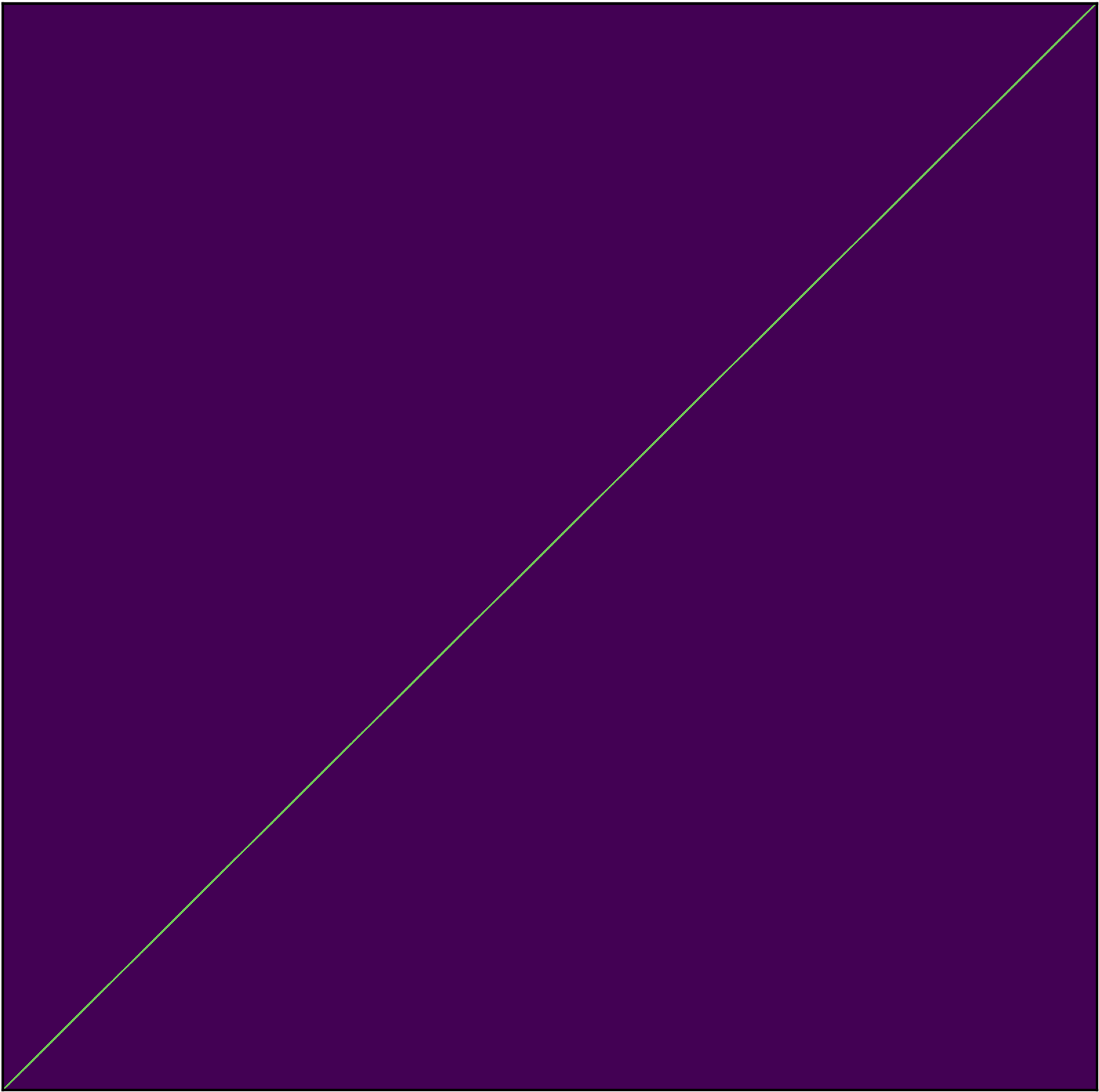}
    \end{subfigure}
    ~
    \centering
    \begin{subfigure}[b]{0.1\textwidth}
        \includegraphics[width=\textwidth]{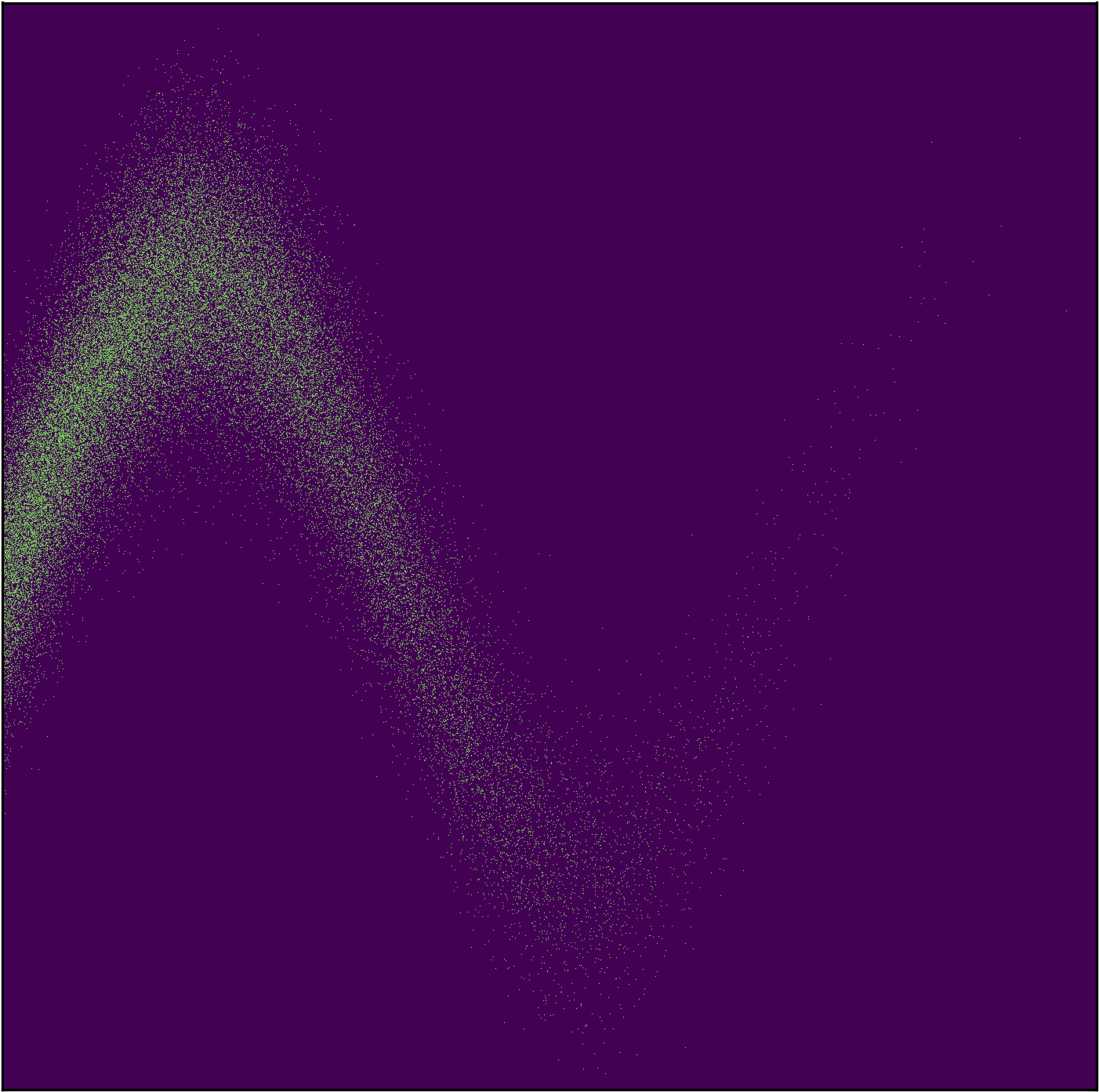}
    \end{subfigure}
    ~
    \centering
    \begin{subfigure}[b]{0.1\textwidth}
        \includegraphics[width=\textwidth]{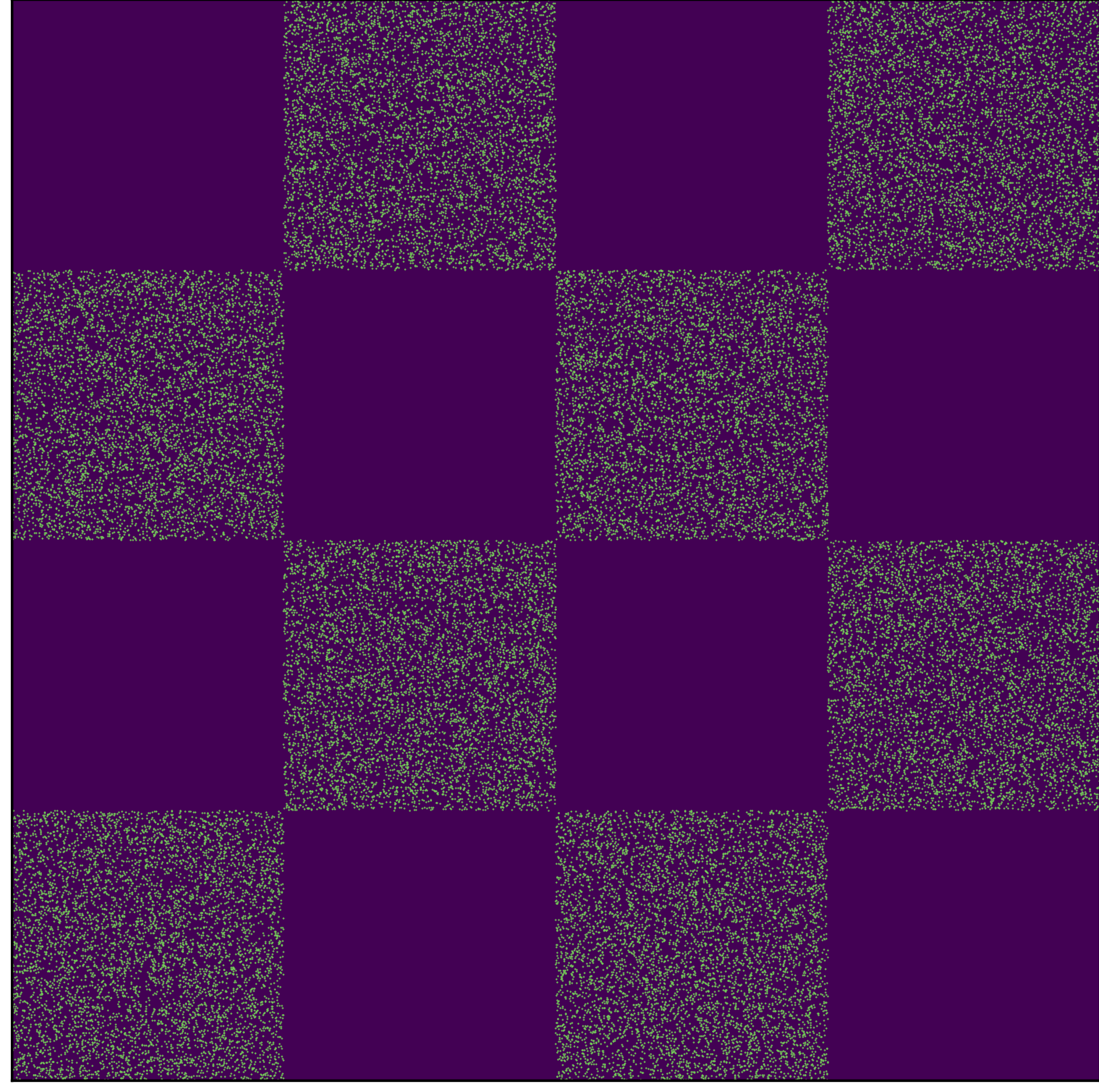}
    \end{subfigure}
    ~
    \centering
    \begin{subfigure}[b]{0.1\textwidth}
        \includegraphics[width=\textwidth]{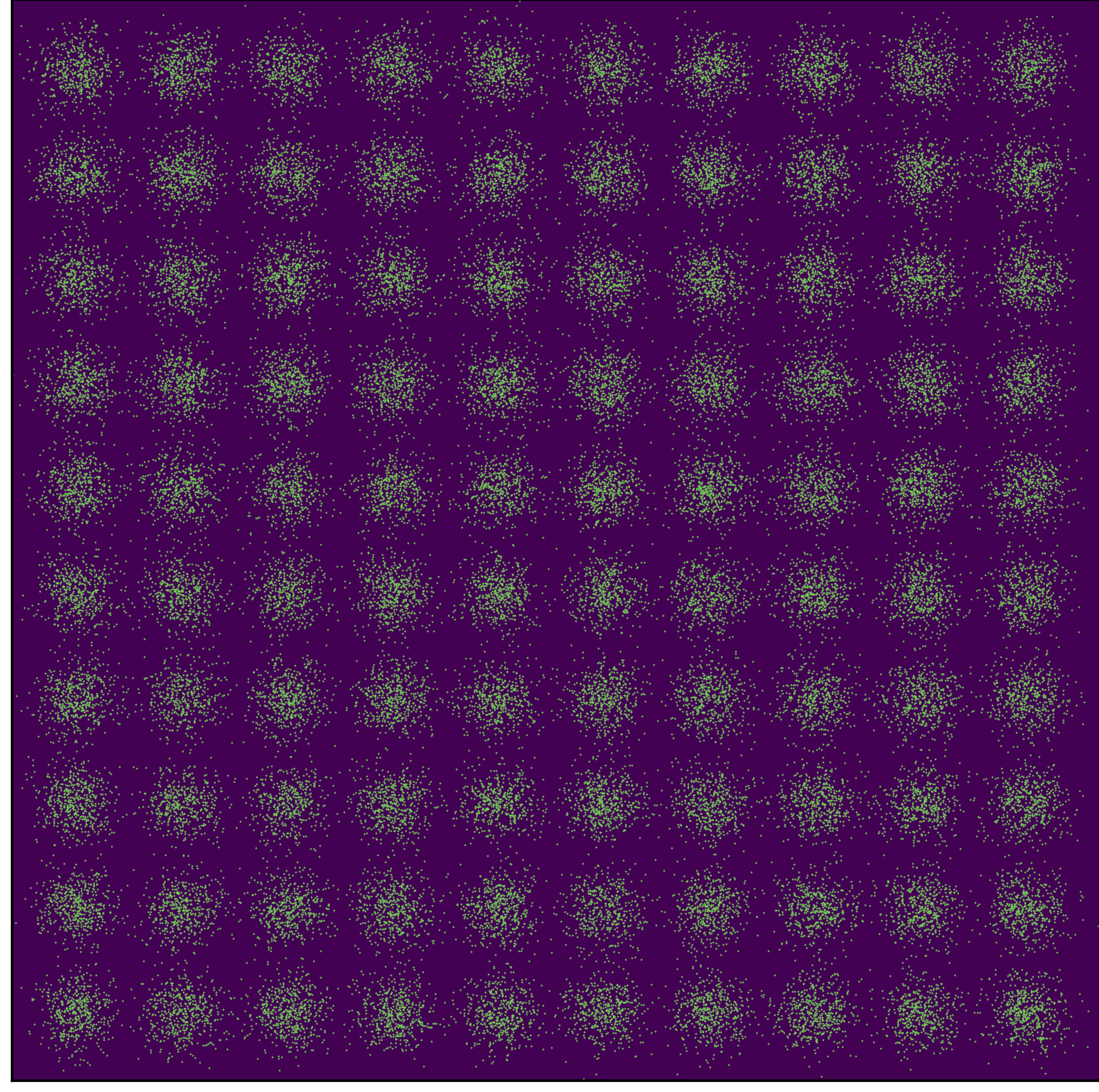}
    \end{subfigure}
    ~
    \begin{subfigure}[b]{0.1\textwidth}
        \includegraphics[width=\textwidth]{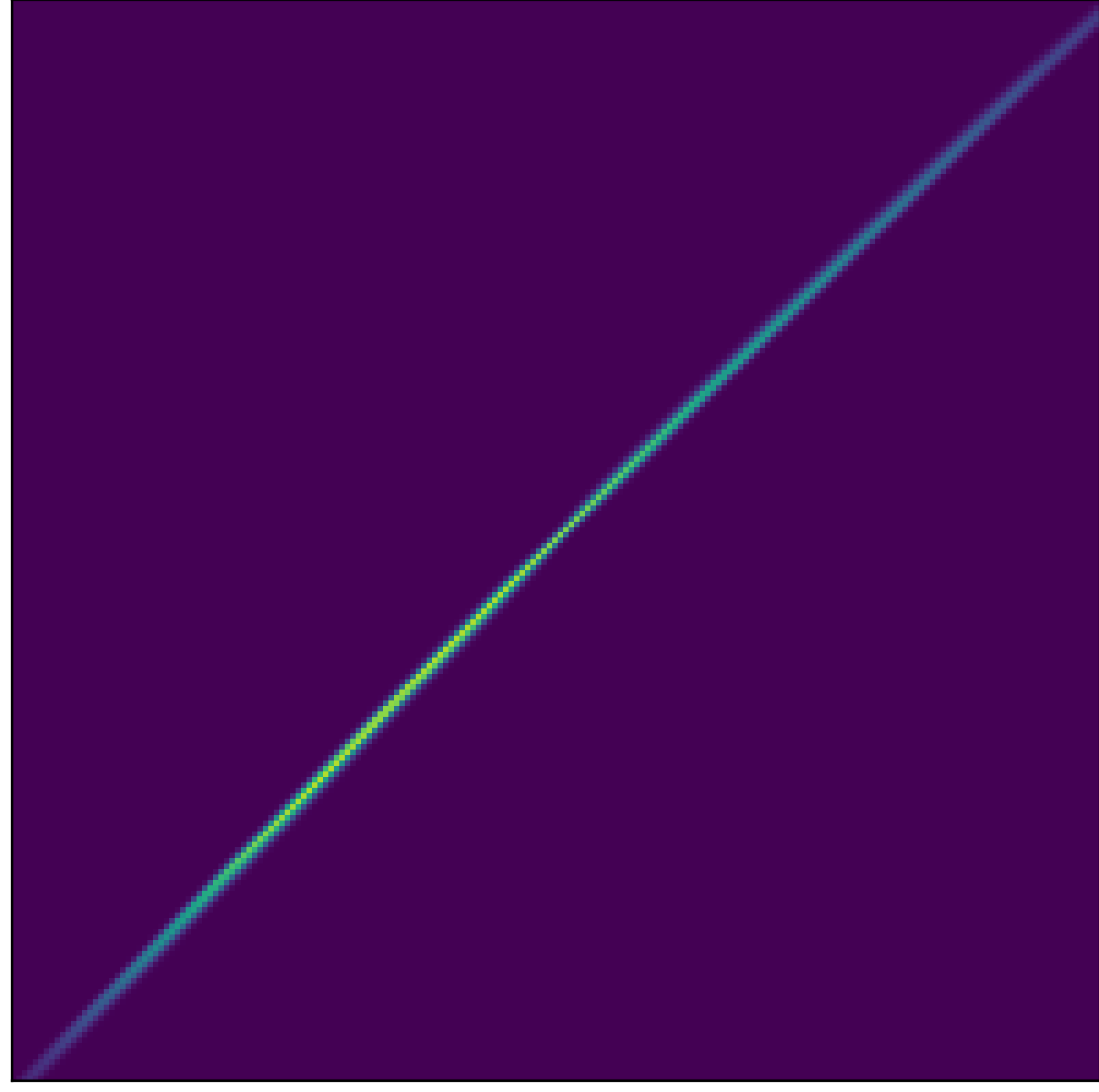}
    \end{subfigure}
    ~
    \begin{subfigure}[b]{0.1\textwidth}
        \includegraphics[width=\textwidth]{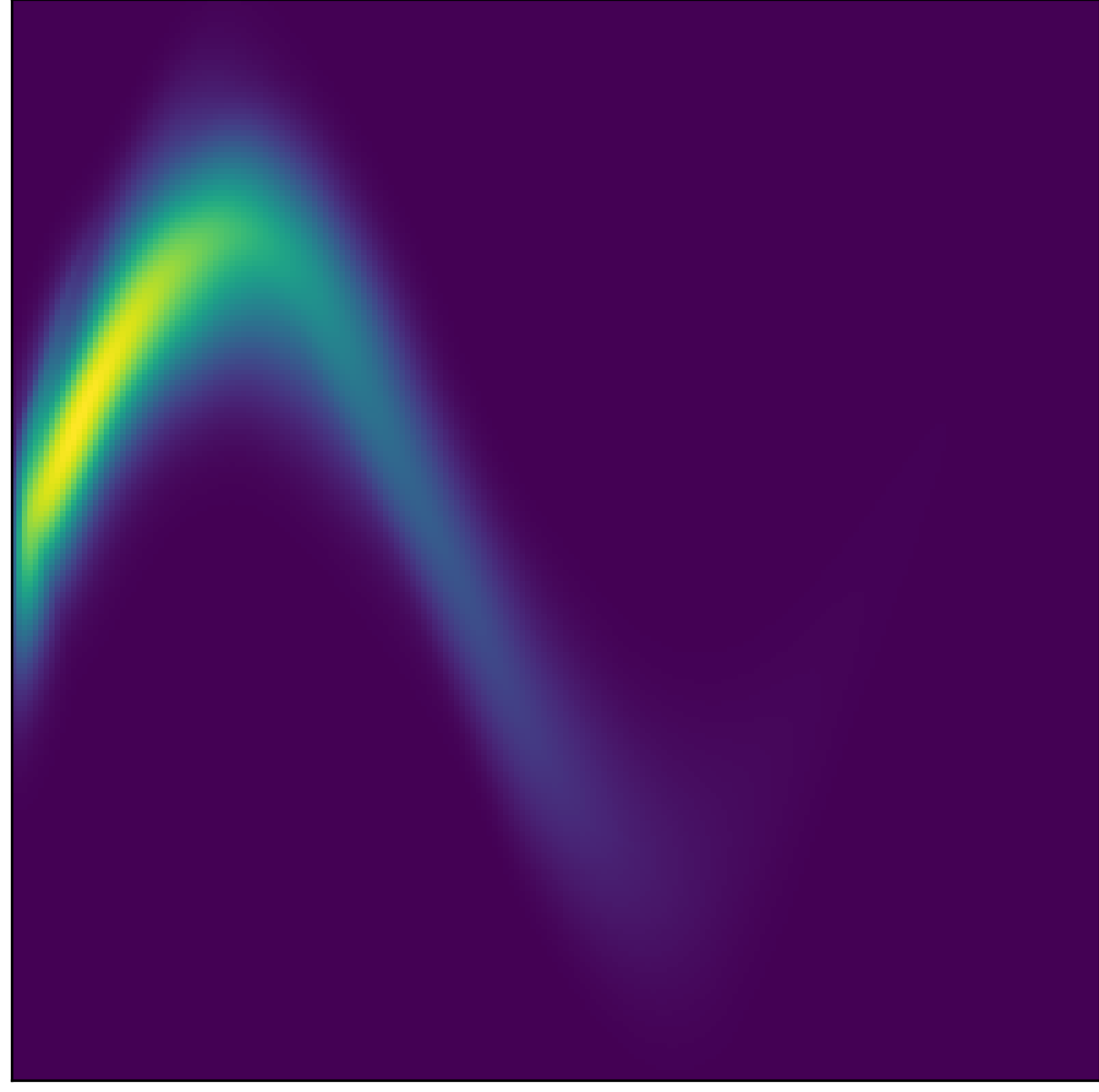}
    \end{subfigure}
    ~
    \begin{subfigure}[b]{0.1\textwidth}
        \includegraphics[width=\textwidth]{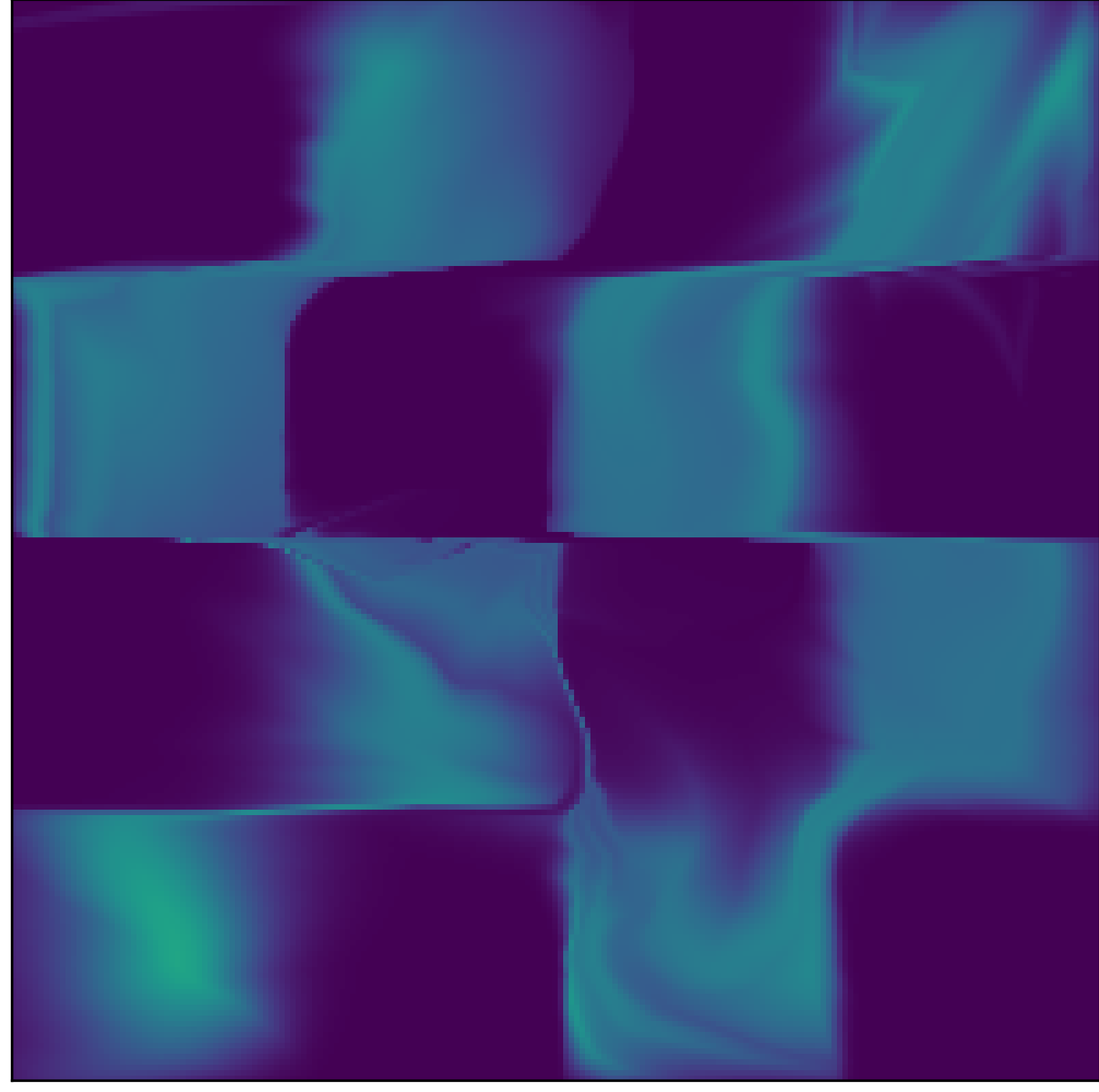}
    \end{subfigure}
    ~
    \begin{subfigure}[b]{0.1\textwidth}
        \includegraphics[width=\textwidth]{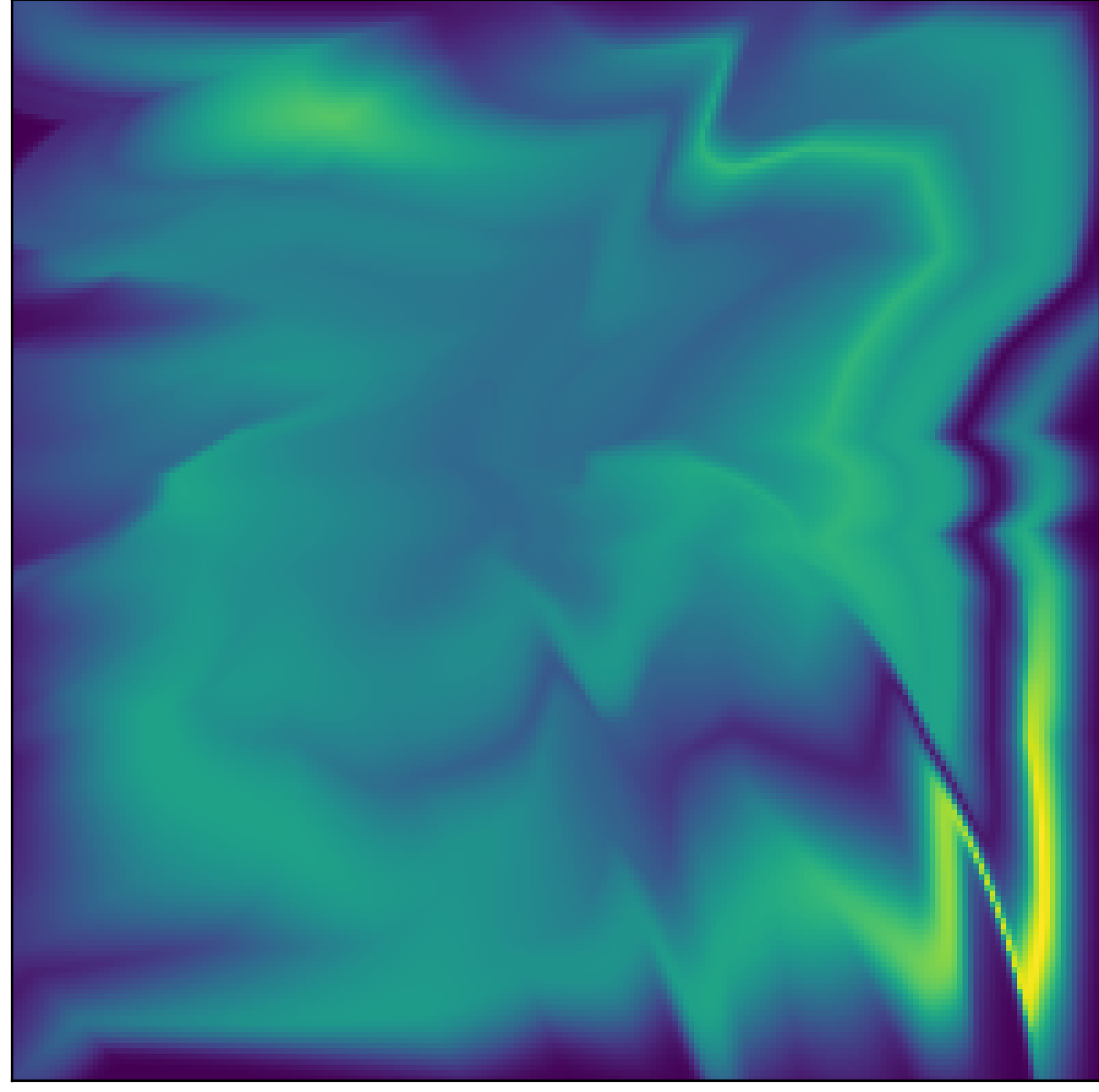}
    \end{subfigure}
    ~ 
    \begin{subfigure}[b]{0.1\textwidth}
        \includegraphics[width=\textwidth]{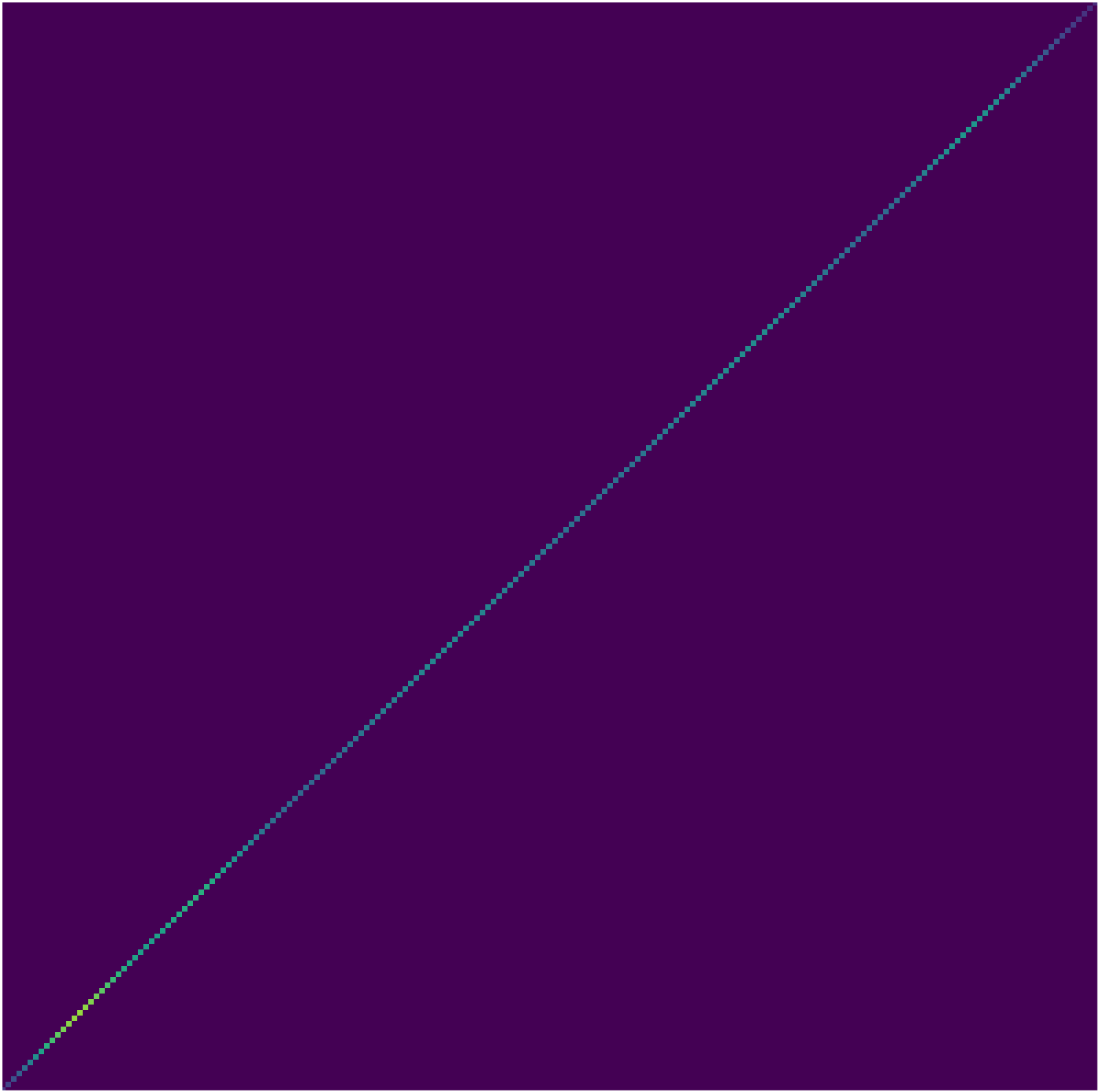}
    \end{subfigure}
    ~
    \begin{subfigure}[b]{0.1\textwidth}
        \includegraphics[width=\textwidth]{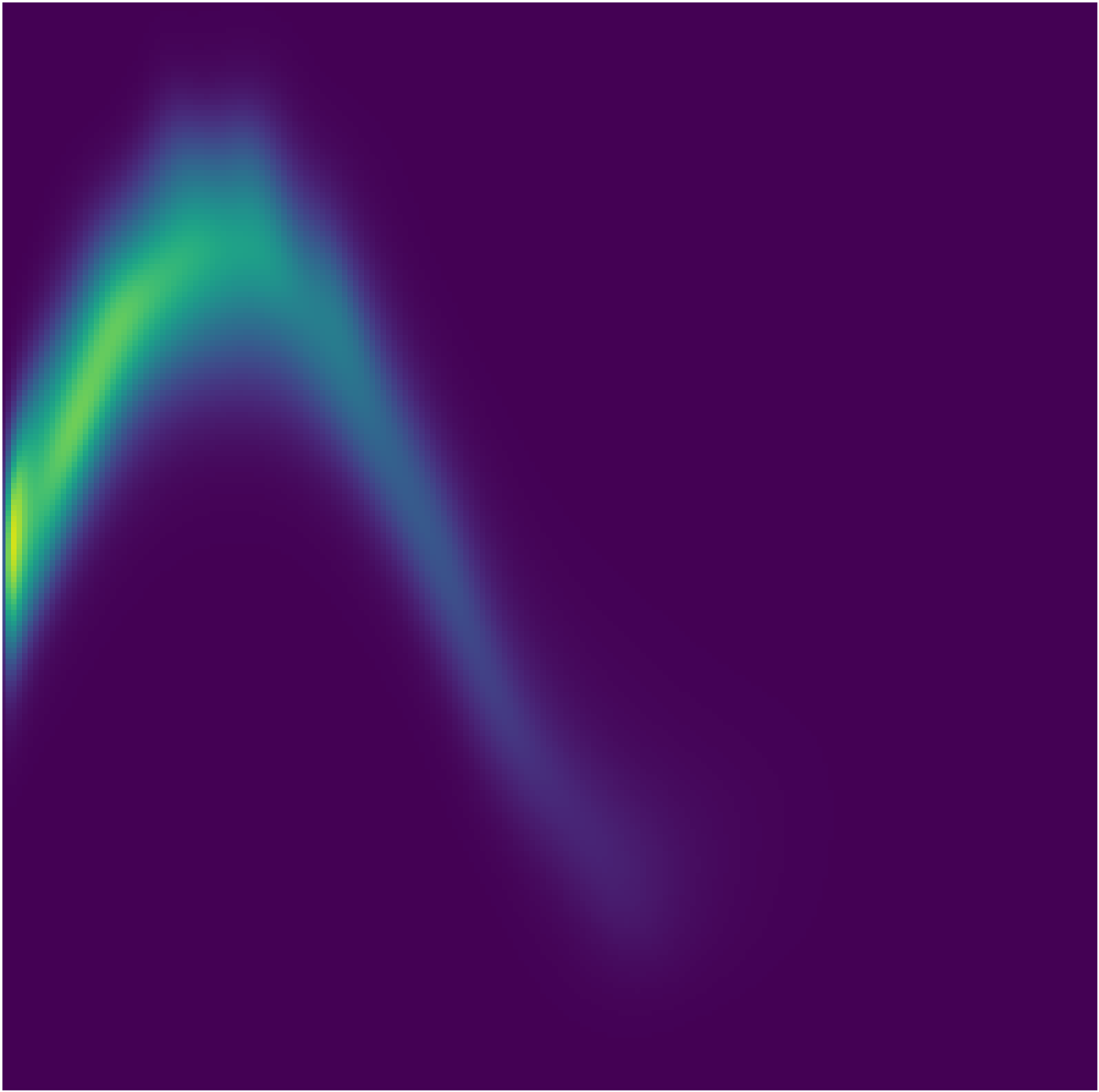}
    \end{subfigure}
    ~
    \begin{subfigure}[b]{0.1\textwidth}
        \includegraphics[width=\textwidth]{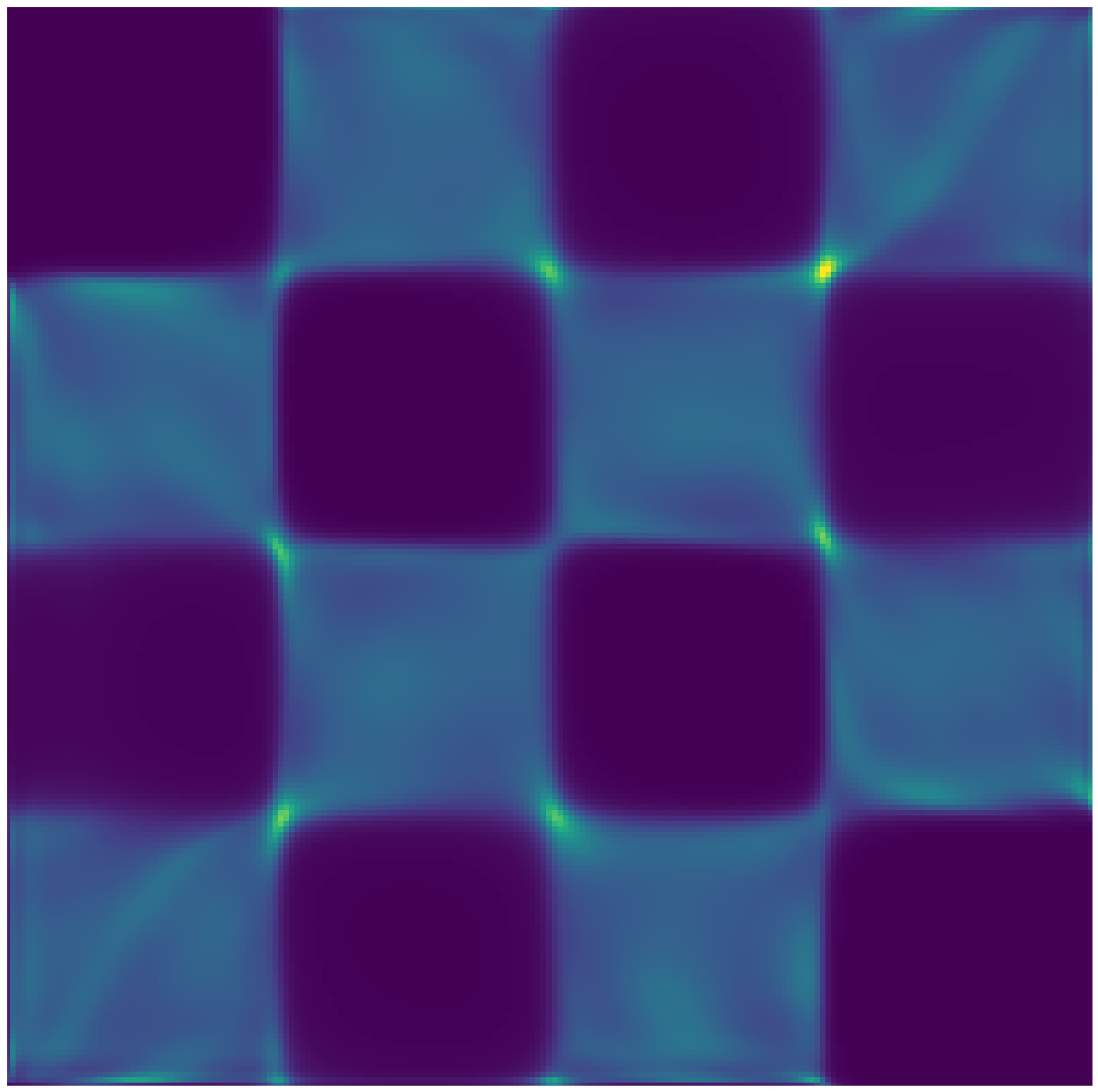}
    \end{subfigure}
    ~
    \begin{subfigure}[b]{0.1\textwidth}
        \includegraphics[width=\textwidth]{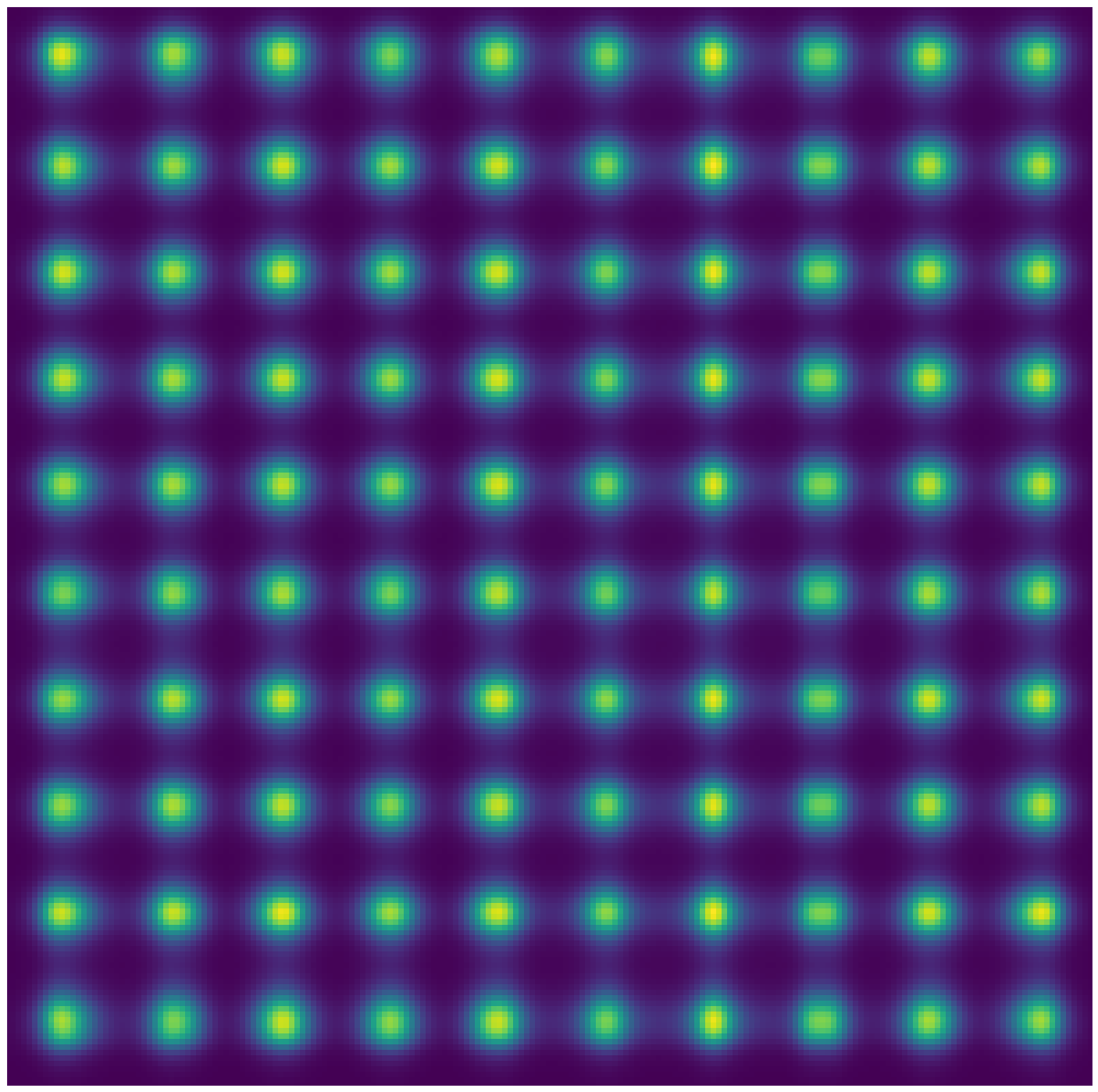}
    \end{subfigure}
    
    \caption{2D density estimation results. \textbf{Top:} Ground truth samples. \textbf{Middle:} Glow. \textbf{Bottom:}
    GF. 
    } 
    \label{fig:toy_samples}
\end{figure}

Our experiments aim to answer the following questions:
\begin{itemize}
    \item Is GF competitive against other methods in terms of density estimation (\ref{sec:exp-toy}, \ref{sec:exp-tab})?
    \item Does GF have better initialization than other normalizing flow models?
    \item Is GF robust against re-parameterization of the data with simple transformations (\ref{sec:exp-strech})?
    \item Does GF achieve good performance when the training set is small (\ref{sec:exp-subset})?
\end{itemize}

\begin{figure*}%
    \centering
        \includegraphics[width=0.8\textwidth]{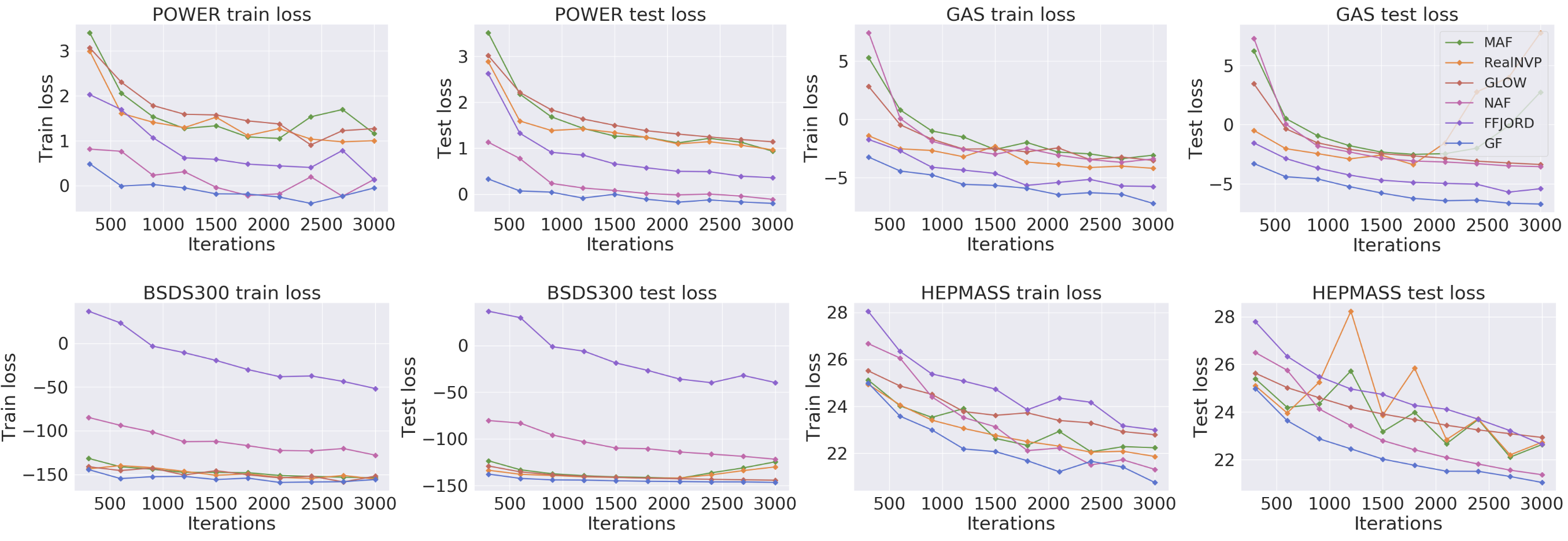}
    \caption{Negative log-likelihood (loss in nats) on training and test sets over initial training iterations. 
    }
    
    \label{fig:initialization}
\end{figure*}

\begin{table*}
 \caption{Negative log-likelihood for tabular datasets measured in nats, and image datasets measured in bpd. Smaller values are better. %
 } \label{tab:tabular_table}
\begin{center}
	\begin{adjustbox}{max width=0.85\textwidth}

    \begin{tabular}{p{2.6cm} |c c c c c |c c c}
        \toprule
        Method &POWER & GAS &HEPMASS &MINIBOONE &BSDS300 & MNIST & FMNIST\\
        \midrule
        Real NVP &-0.17 &-8.33 & 18.71 &13.55 &-153.28 &1.06 &\textbf{2.85}\\
        Glow &-0.17 &-8.15 & 18.92 &11.35 &-155.07 &1.05 &2.95\\
        FFJORD &-0.46 &-8.59 &\textbf{14.92} &10.43 &\textbf{-157.40} &\textbf{0.99} &-\\
        RBIG &1.02 &0.05 & 24.59 &25.41 &-115.96 &1.71 &4.46\\
        GF(ours) &\textbf{-0.57} &\textbf{-10.13} & 17.59 &\textbf{10.32} &-152.82 & 1.29 & 3.35\\
        \midrule
        MADE &3.08 &-3.56 &20.98 &15.59 &-148.85 &2.04 &4.18\\
        MAF &-0.24 &-10.08 &17.70 &11.75 &-155.69 &1.89 &-\\
        TAN &-0.48 &-11.19 &15.12 &11.01 &-157.03 &- &-\\
        MAF-DDSF &-0.62 &-11.96 &15.09 &8.86 &-157.73 &- &-\\
        \bottomrule
    \end{tabular} 
    \end{adjustbox}
\end{center}
\end{table*}

\subsection{2D Toy Datasets}
\label{sec:exp-toy}
We first perform density estimation on four synthetic datasets drawn from complex two-dimensional distributions with various shapes and number of modes. We train the model by warping the predicted probability distribution to an isotropic Gaussian distribution. In \figref{fig:toy_samples}, we visualize the estimated density of our Gaussianization Flow and Glow. The results show that our model is capable of fitting both continuous and discontinuous, connected and disconnected multi-modal distributions. Glow, on the other hand, has trouble modeling disconnected distributions.

\subsection{Tabular and Image Datasets}
\label{sec:exp-tab}
We perform density estimation on five tabular datasets which are preprocessed using the method in \citet{maf}. We compare our results directly with RealNVP, Glow and FFJORD as these are also %
efficiently invertible models which can be used  for sample generation and inference; we list MAF, MADE, TAN and NAF results as reference as they have higher computational costs in sampling but are competitive in density estimation. 
From \tabref{tab:tabular_table}, we observe that GF achieves the top negative log-likelihood results in 3 out of 5 tabular datasets, and obtain comparable results on the remaining two. As expected, Gaussianization flow outperforms RBIG on all tasks by a large margin, which demonstrates the strong advantages of joint training by maximum likelihood.

For tabular datasets, %
we use $D$ Householder reflections for each \textit{trainable rotation matrix layer} where $D$ equals the data dimension, so that the model possesses the ability to parameterize all possible rotation matrices. See \appref{sec:settings} for more training details. 

We also consider two image datasets, MNIST and Fashion-MNIST, and perform density estimation on the continuous distribution of uniformly dequantized images (see \tabref{tab:tabular_table}). For image data, we use patch-based rotation matrices as \textit{trainable rotation matrix layers}. Specifically, we set the patch size to 4 and randomly pick the shifting constant $c$ at each layer. We provide more training details in Appendix \ref{sec:settings}. From the results in \tabref{tab:tabular_table} we see that Gaussianization flow outperforms all other non-convolutional models on image datasets, including those that cannot be inverted efficiently, such as MAF and MADE~\citep{made}.

\subsection{Initial Performance}
The data-dependent initialization of our model allows the training process to converge faster.
To illustrate this, we choose four tabular datasets (pre-processed as described in~\citet{maf}), where we set the batch size to be 500 and perform training for 3000 iterations using the default settings and model architectures. From the results in \figref{fig:initialization}, Gaussianization flow achieves better training and validation performance across most iterations on the four datasets compared with other models such as RealNVP, Glow, FFJORD, MAF and NAF.

\subsection{Stretched Tabular Datasets}
\label{sec:exp-strech}
\begin{table*}

 \caption{Negative log-likelihood in nats for tabular datasets after simple transformations. ``$*$'' stands for loss larger than 1000. ``$**$'' implies loss does not converge and varies largely on different batches. ``-'' implies loss explosion on validation and test sets. ``NaN'' implies numerical issues encountered during training. Numbers in parentheses for GF denote the corresponding likelihood value under the original normalized transformation.
 } 
 \label{tab:transformation_table}
\begin{center}
\begin{adjustbox}{max width=0.9\textwidth}
    \begin{tabular}{l |c c c| c c c}
        \toprule
        Transformation&\multicolumn{3}{c|}{$f(x)=x^3$} & \multicolumn{3}{c}{$f(x)=1000x+51$} \\
        \midrule
        Method &POWER &MINIBOONE &GAS &POWER &MINIBOONE &GAS\\
        \midrule
        Real NVP &17.47~(21.53) &93.98~(109.96) &32.27~(32.85)  &- &- &-\\
        Glow &1.67~(5.73) &91.86~(107.84) &- &41.64~(0.19) &\text{315.30~(18.27)} &49.26~(-6.00)\\
        FFJORD &$*$ &88.29~(104.27) &$**$ &$*$ &329.97~(32.94) &$*$\\
        GF(ours) &\textbf{-4.41~(-0.35)} &\textbf{4.62~(20.60)} &\textbf{-6.91~(-6.33)} & \textbf{41.00(-0.45)} &325.72~(28.69) &\textbf{47.69~(-7.57)}\\
        \midrule
        MAF &19.37~(23.43) &381.32~(397.3) &19.76~(20.34) &- &- &- \\
        MAF-DDSF &-4.12~(-0.06) &7.88~(23.86) &-4.52~(-3.94) &NaN &NaN &NaN\\
        \bottomrule
    \end{tabular} 
\end{adjustbox}
\end{center}
\end{table*}

\begin{figure*}%
    \centering
    \includegraphics[width=0.9\textwidth]{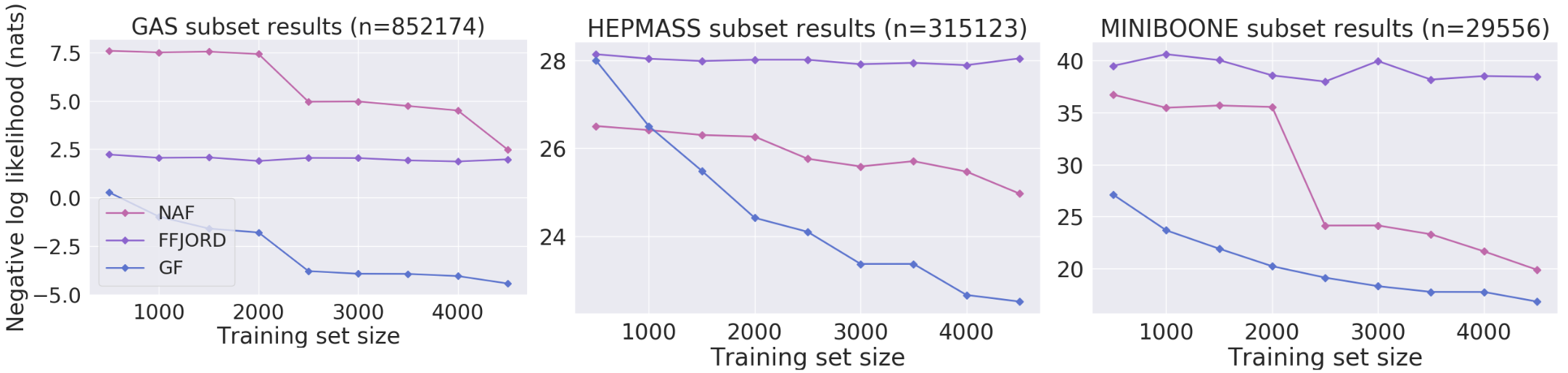}
    \caption{Negative log-likelihood results (measured in nats) over small subsets of the original training set. The subsets, with size ranging from $500$ to $4500$, are much smaller than the original training set size $n$ as shown in the parentheses. We exclude MAF, RealNVP and Glow in the figures as validation error does not decrease. 
    }
    \label{fig:subset}
\end{figure*}

In density estimation applications (such as anomaly detection), one could receive a stream of data that might not be sampled \iid from a fixed distribution. In these cases, it would be difficult to find suitable pre-processing techniques to normalize the data, so it is desirable if our models can be robust under distributions that are not normalized.
To evaluate whether the flow models are robust against certain distribution shifts (that could make normalization difficult), we consider density estimation on datasets that are not normalized. In particular, we select three pre-processed UCI datasets and transform the data using some simple invertible transformations before training. Here we keep the transformations invertible and differentiable so that we can use the change of variables formula to compute the likelihoods defined in the original data space. We consider two transformations: \textit{cubic}, where $f(x) = x^3$; and \textit{affine} where $f(x) = 1000x + 51$.

From the results in \tabref{tab:transformation_table}, we observe that Gaussianization flows have stable and consistent performance for both transformations and on all three datasets. In contrast, all other methods can fail in some settings. MAF-DDSF has numerical issues that lead to NaNs on datasets processed with the \textit{affine} transformation; RealNVP, Glow, and MAF all have cases where test loss does not go down when training loss goes down; FFJORD has convergence issues when training on GAS with the \textit{cubic} transformation, and on POWER and GAS with the \textit{affine} transformation. Moreover, even with the added transformation, we are still able to obtain comparable likelihood when transformed back to the original space (see \tabref{tab:transformation_table} results in parentheses).%

\subsection{Small Training Sets}%
\label{sec:exp-subset}
The ability to quickly adapt to new distributions with relatively few samples (\eg in a stream of data with continuous covariate shifts) can also be helpful. To this end, we further evaluate the generalization abilities of the models when trained on small subsets of the tabular datasets. We consider using the normalized tabular datasets, where we mix the training, validation and test datasets, shuffle them randomly and select 10,000 samples as validation/test sets respectively. %
We consider training various model on small training subsets with sizes ranging from 500 to 4500, where we perform validation and testing on the new validation/test sets. 

We compare GF with Glow, RealNVP, MAF, FFJORD and NAF, %
using the same model architecture for the original tabular experiments and explore the learning rate to make the training process more stable. We show the results in \figref{fig:subset}. We note that MAF, Glow and RealNVP have trouble evaluating density on validation/test set when the training set is small enough, as the validation/test loss goes up as train loss goes down, which is the reason why we exclude them in the plots. GF significantly outperforms FFJORD and NAF in all settings except when subset size is 500 for HEPMASS, which suggests that our learnable KDE layers generalize well on test sets even when training data is scarce.

\section{CONCLUSION}
We introduce Gaussianization flows (GF), a new family of trainable flow models that builds upon rotation-based iterative Gaussianization. GFs exhibit fast likelihood evaluation and fast sample generation, and are expressive enough to be universal approximators for most continuous probability distributions. Empirical results demonstrate that GFs achieve better or comparable performance against existing state-of-the-art flow models that are efficiently invertible. Compared to other efficiently invertible models, GFs have better initializations, are more robust to distribution shifts in training data, and have superior generalization when training data are scarce. Combining the advantages of GFs with other efficiently invertible flow models would be an interesting direction for future research.

\subsubsection*{Acknowledgements}
This research was supported by  Amazon AWS, TRI, NSF (\#1651565, \#1522054, \#1733686), ONR (N00014-19-1-2145), AFOSR (FA9550-19-1-0024).

\bibliography{gaussianization}

\begin{thebibliography}{}

\bibitem[Chen and Gopinath, 2001]{chen2001gaussianization}
Chen, S.~S. and Gopinath, R.~A. (2001).
\newblock Gaussianization.
\newblock In {\em Advances in neural information processing systems}, pages
  423--429.

\bibitem[Chen et~al., 2018]{chen2018neural}
Chen, T.~Q., Rubanova, Y., Bettencourt, J., and Duvenaud, D.~K. (2018).
\newblock Neural ordinary differential equations.
\newblock In Bengio, S., Wallach, H., Larochelle, H., Grauman, K.,
  Cesa-Bianchi, N., and Garnett, R., editors, {\em Advances in Neural
  Information Processing Systems 31}, pages 6571--6583. Curran Associates, Inc.

\bibitem[De~Cao et~al., 2019]{de2019block}
De~Cao, N., Titov, I., and Aziz, W. (2019).
\newblock Block neural autoregressive flow.
\newblock {\em arXiv preprint arXiv:1904.04676}.

\bibitem[Deng et~al., 2009]{imagenet_cvpr09}
Deng, J., Dong, W., Socher, R., Li, L.-J., Li, K., and Fei-Fei, L. (2009).
\newblock {ImageNet: A Large-Scale Hierarchical Image Database}.
\newblock In {\em CVPR09}.

\bibitem[Devroye and Wagner, 1979]{devroye1979}
Devroye, L.~P. and Wagner, T.~J. (1979).
\newblock The $l_1$ convergence of kernel density estimates.
\newblock {\em Ann. Statist.}, 7(5):1136--1139.

\bibitem[Dinh et~al., 2014]{dinh2014nice}
Dinh, L., Krueger, D., and Bengio, Y. (2014).
\newblock Nice: Non-linear independent components estimation.
\newblock {\em arXiv preprint arXiv:1410.8516}.

\bibitem[Dinh et~al., 2015]{dinh2016density}
Dinh, L., Krueger, D., and Bengio, Y. (2015).
\newblock {NICE:} non-linear independent components estimation.
\newblock In {\em 3rd International Conference on Learning Representations,
  {ICLR} 2015, San Diego, CA, USA, May 7-9, 2015, Workshop Track Proceedings}.

\bibitem[Dinh et~al., 2016]{nvp}
Dinh, L., Sohl-Dickstein, J., and Bengio, S. (2016).
\newblock Density estimation using real nvp.
\newblock {\em arXiv preprint arXiv:1605.08803}.

\bibitem[Germain et~al., 2015]{made}
Germain, M., Gregor, K., Murray, I., and Larochelle, H. (2015).
\newblock Made: Masked autoencoder for distribution estimation.
\newblock {\em International Conference on Machine Learning}, 37:881--889.

\bibitem[Grathwohl et~al., 2018]{FFJORD}
Grathwohl, W., Ricky T. Q.~Chen, Jesse~Bettencourt, I.~S., and Duvenaud, D.
  (2018).
\newblock Ffjord: Free-form continuous dynamics for scalable reversible
  generative models.
\newblock {\em arXiv preprint arXiv:1810.01367}.

\bibitem[Huang et~al., 2018]{huang2018neural}
Huang, C.-W., Krueger, D., Lacoste, A., and Courville, A. (2018).
\newblock Neural autoregressive flows.
\newblock In {\em International Conference on Machine Learning}, pages
  2083--2092.

\bibitem[Huber, 1985]{huber1985projection}
Huber, P.~J. (1985).
\newblock Projection pursuit.
\newblock {\em The annals of Statistics}, pages 435--475.

\bibitem[Kingma and Dhariwal, 2018]{glow}
Kingma, D.~P. and Dhariwal, P. (2018).
\newblock Glow: Generative flow with invertible 1x1 convolutions.
\newblock {\em arXiv preprint arXiv:1807.03039}.

\bibitem[Krizhevsky et~al., 2009]{krizhevsky2009learning}
Krizhevsky, A. et~al. (2009).
\newblock Learning multiple layers of features from tiny images.
\newblock Technical report, Citeseer.

\bibitem[Laparra et~al., 2011]{laparra2011iterative}
Laparra, V., Camps-Valls, G., and Malo, J. (2011).
\newblock Iterative gaussianization: from ica to random rotations.
\newblock {\em IEEE transactions on neural networks}, 22(4):537--549.

\bibitem[Oliva et~al., 2018]{oliva2018transformation}
Oliva, J.~B., Dubey, A., Zaheer, M., Póczos, B., Salakhutdinov, R., Xing,
  E.~P., and Schneider, J. (2018).
\newblock Transformation autoregressive networks.

\bibitem[Papamakarios et~al., 2017]{maf}
Papamakarios, G., Pavlakou, T., and Murray, I. (2017).
\newblock Masked autoregressive flow for density estimation.
\newblock In {\em Advances in Neural Information Processing Systems}, pages
  2338--2347.

\bibitem[Parzen, 1962]{parzen1962estimation}
Parzen, E. (1962).
\newblock On estimation of a probability density function and mode.
\newblock {\em The annals of mathematical statistics}, 33(3):1065--1076.

\bibitem[Rezende and Mohamed, 2015]{rezende15variational}
Rezende, D. and Mohamed, S. (2015).
\newblock Variational inference with normalizing flows.
\newblock In Bach, F. and Blei, D., editors, {\em Proceedings of the 32nd
  International Conference on Machine Learning}, volume~37 of {\em Proceedings
  of Machine Learning Research}, pages 1530--1538, Lille, France. PMLR.

\bibitem[Scott and Sheather, 1985]{scott1985kernel}
Scott, D.~W. and Sheather, S.~J. (1985).
\newblock Kernel density estimation with binned data.
\newblock {\em Communications in Statistics-Theory and Methods},
  14(6):1353--1359.

\bibitem[Sheather, 2004]{sheather2004density}
Sheather, S.~J. (2004).
\newblock Density estimation.
\newblock {\em Statistical science}, pages 588--597.

\bibitem[Tomczak and Welling, 2016]{tomczak2016improving}
Tomczak, J.~M. and Welling, M. (2016).
\newblock Improving variational auto-encoders using householder flow.
\newblock {\em arXiv preprint arXiv:1611.09630}.

\end{thebibliography}
\allowdisplaybreaks
\appendix
\FloatBarrier
\onecolumn
\section{PROOFS}\label{app:proof}
\subsection{Mixtures of Logistics are Universal Approximators}
First, we show that mixtures of logistics are universal approximators for any univariate continuous densities supported on a compact set. The proof is based on the classic result on uniform convergence of kernel density estimation (KDE):
\begin{lemma}[Uniform convergence of KDE~\citep{parzen1962estimation}]\label{lem:universal_kde}
Let $f$ be a continuous probability density function on a compact set $\mcal{X} \subset \mbb{R}$. Let $K(y)$ denote a symmetric probability density function over $\mbb{R}$ whose characteristic function is absolutely integrable. Let $\{x_1, x_2, \cdots, x_n\}$ be \iid samples from $f(x)$. Given any $h_n \to 0$ such that $h_n\sqrt{n} \to \infty$, we may define
\begin{align*}
    \hat{f}_n(x) = \frac{1}{nh_n}\sum_{i=1}^n K\left(\frac{x - x_i}{h_n}\right)
\end{align*}
as an approximation to $f(x)$ over $\mcal{X}$ with high probability; that is
\begin{align*}
    \max_{x\in\mcal{X}} |\hat{f}_n(x) - f(x)| \stackrel{p}{\to} 0,\quad \text{as $n\to\infty$}.
\end{align*}
\end{lemma}

The universal approximation of mixtures of logistics is an immediate result of \lemref{lem:universal_kde}.

\begin{lemma}[Mixtures of logistics are universal approximators to probability densities]\label{lem:universal}
Let $\mcal{X}$ be any compact subset of $\mbb{R}$. Let $\mcal{P}(\mcal{X})$ denote the space of continuous probability densities on $\mcal{X}$. Then, given any $\epsilon > 0$ and any probability density $p \in \mcal{P}(\mcal{X})$, there exist an integer $N$, real constants $\mu_i \in \mbb{R}, \sigma_i \in \mbb{R}^{++}$ for $i =1,\cdots,N$ such that we may define 
\begin{align*}
    f_N(x) = \frac{1}{N} \sum_{i=1}^N \frac{e^{-\frac{x-\mu_i}{\sigma_i}}}{\sigma_i \left(1 + e^{-\frac{x-\mu_i}{\sigma_i}} \right)^2}
\end{align*}
as an approximate realization of the probability density $p$ over $\mcal{X}$; that is
\begin{align*}
    \sup_{x\in\mcal{X}} |f_N(x) - p(x)| < \epsilon.
\end{align*}
In other words, functions of the form $f_N(x)$ are dense in $\mcal{P}(\mcal{X})$.
\end{lemma}
\begin{proof}
We note that $f_N(x)$ is a kernel density estimator on data points $\{\mu_i\}_{i=1}^N$, with the logistic kernel
\begin{align*}
    K(y) = \frac{e^{-y}}{(1+e^{-y})^2}.
\end{align*}
It is a continuous probability distribution (logistic distribution), symmetric, and has an absolutely integrable characteristic function. Using \lemref{lem:universal_kde} we conclude that for any $\epsilon > 0$ and $0 < \delta < 1$, there exists $N$ and $h_N$ such that
\begin{align*}
    \Pr_{X_1, \cdots, X_N \sim p(X)}\left[\sup_{x\in\mcal{X}} \left|  \frac{1}{N h_N} \sum_{i=1}^N \frac{e^{-\frac{x-X_i}{h_N}}}{ \left(1 + e^{-\frac{x-X_i}{h_N}} \right)^2} - p(x) \right| < \epsilon \right] > 1 - \delta > 0.
\end{align*}
Since the probability is non-zero, there exists $x_1, x_2, \cdots, x_N$ such that
\begin{align*}
    \sup_{x\in\mcal{X}}\left| \frac{1}{N} \sum_{i=1}^N \frac{e^{-\frac{x-X_i}{h_N}}}{h_N \left(1 + e^{-\frac{x-X_i}{h_N}} \right)^2} - p(x) \right| < \epsilon.
\end{align*}
Letting $\mu_i = x_i$ and $\sigma_i = h_N$ for $i =1,2,\cdots, N$, we have
\begin{align*}
    \sup_{x\in\mcal{X}} |f_N(x) - p(x)| < \epsilon,
\end{align*}
which proves our statement.
\end{proof}

\begin{lemma}\label{lem:universal_kl}
    Let $p$ be a continuous probability distribution on a compact set $\mcal{X} \subset \mbb{R}$. Assume there exists $\delta > 0$ such that $\inf_{x\in\mcal{X}} p(x) > \delta$. Then, given any $\epsilon > 0$, there exist an integer $N$, real constants $\mu_i \in \mbb{R}, \sigma_i \in \mbb{R}^{++}$ for $i =1,\cdots,N$ such that we may define
    \begin{align*}
        f_N(x) = \frac{1}{N} \sum_{i=1}^N \frac{e^{-\frac{x-\mu_i}{\sigma_i}}}{\sigma_i \left(1 + e^{-\frac{x-\mu_i}{\sigma_i}} \right)^2}
    \end{align*}
    to approximate $p(x)$ in terms of KL divergence; that is
    \begin{align*}
        \KLD{p(x)}{f_N(x)} < \epsilon.
    \end{align*}
\end{lemma}
\begin{proof}
    It suffices to show that
    \begin{align*}
        \left| \int_{x\in\mcal{X}} p(x) \log f_N(x)\ud x - \int_{x\in\mcal{X}} p(x) \log p(x)\ud x\right| < \epsilon.
    \end{align*}
    Using \lemref{lem:universal} we know that
    \begin{align*}
        \sup_{x\in\mcal{X}} |f_N(x) - p(x)| \to 0, \quad \text{as $N \to \infty$}.
    \end{align*}
    For $\epsilon_1 > 0$, there exists an integer $N$, real constants $\mu_i \in \mbb{R}, \sigma_i \in \mbb{R}^{++}$ for $i =1,\cdots,N$ such that
    \begin{align*}
        \sup_{x\in\mcal{X}} |f_N(x) - p(x)| < \epsilon_1.
    \end{align*}
    Because $p(x) > \delta$, we know $f_N(x) > \delta - \epsilon_1$, in which case 
    \begin{align*}
        \sup_{x\in\mcal{X}}|\log f_N(x) - \log p(x)| \leq \left(\sup_{t \geq \delta - \epsilon_1} \frac{1}{t}\right)\left(\sup_{x\in\mcal{X}} |f_N(x) - p(x)|\right) \leq \frac{\epsilon_1}{\delta - \epsilon_1}.
    \end{align*}
    Then,
    \begin{align*}
        \left| \int_{x\in\mcal{X}} p(x) \log f_N(x)\ud x - \int_{x\in\mcal{X}} p(x) \log p(x)\ud x\right| &\leq \int_{x\in\mcal{X}}p(x)|\log f_N(x) - \log p(x)|\ud x \\
        &\leq \int_{x\in\mcal{X}} p(x) \sup_{x\in\mcal{X}}|\log f_N(x) - \log p(x)|\ud x\\
        &\leq \frac{\epsilon_1}{\delta - \epsilon_1}.
    \end{align*}
    Our statement is proved by setting $\epsilon_1 = \frac{\epsilon \delta}{1 + \epsilon}$.
\end{proof}
\begin{corollary}\label{cor:kl}
    Let $p$ be a continuous probability distribution on a compact set $\mcal{X} \subset \mbb{R}$. Assume there exists $\delta > 0$ such that $\inf_{x\in\mcal{X}} p(x) > \delta$. Then, given any $\epsilon > 0$, there exist an integer $N$, real constants $\mu_i \in \mbb{R}, \sigma_i \in \mbb{R}^{++}$ for $i =1,\cdots,N$ such that if we define
        \begin{align*}
            F_N(x) = \frac{1}{N} \sum_{i=1}^N \sigma\left( \frac{x - \mu_i}{\sigma_i} \right),
        \end{align*}
        where $\sigma(\cdot)$ is the sigmoid function, and transform random variable $X \sim p$ to $Y = \Phi^{-1}(F_N(X))$ where $\Phi(\cdot)$ is the CDF of standard normal, we have
        \begin{align*}
            \KLD{Y}{\mcal{N}(0,1)} < \epsilon.
        \end{align*}
\end{corollary}
\begin{proof}
    Let $Z\sim \mcal{N}(0,1)$. Since KL divergence is invariant to bijective transformations, we get
    \begin{align*}
        \KLD{Y}{Z} = \KLD{F_N(X)}{\Phi(Z)} = \KLD{X}{F_N^{-1}(\Phi(Z))} \stackrel{(i)}{=} \KLD{p(x)}{f_N(x)},
    \end{align*}
    where $f_N(x) = F_N'(x)$. Here $(i)$ is because $\Phi(Z)$ is a uniform random variable and the inverse CDF trick of producing samples from a distribution with CDF $F_N(x)$ and PDF $f_N(x)$.
\end{proof}

\subsection{Gaussianization Flows are Universal Approximators}

\begin{customthm}{\ref{thm:1}}
Let $p$ be any continuous distribution supported on a compact set $\mcal{X} \subset \mbb{R}^D$, and $\inf_{x\in\mcal{X}} p(x) \geq \delta$ for some constant $\delta > 0$. Let $\Psi: \mbb{R}^D \to \mbb{R}^D$ denote a marginal Gaussianization layer of the Gaussianization flow, where for each element $x_i$ of $\bfx$, we choose an integer $N_i$, real constants $\mu_{ij} \in \mbb{R}, \sigma_{ij} \in \mbb{R}^{++}$ for $j=1,\cdots,N_i$ such that
\begin{align*}
    \Psi_i(\bfx) = \Phi^{-1}\left( \frac{1}{N_i} \sum_{j=1}^{N_i} \sigma\left( \frac{x_i - \mu_{ij}}{\sigma_{ij}} \right) \right).
\end{align*}
Then, there exists a sequence of marginal Gaussianization layers $\{\Psi^{(1)}, \Psi^{(2)},\cdots, \Psi^{(k)}, \cdots\}$ and rotation matrices $\{R^{(1)}, R^{(2)}, \cdots, R^{(k)}, \cdots\}$ such that the transformed random variable
\begin{align*}
    \bfX^{(k)} \triangleq \Psi^{(k)}(R^{(k)} \Psi^{(k-1)}(R^{(k-1)} \cdots \Psi^{(1)}(R^{(1)}\bfX))) \stackrel{d}{\to} \mcal{N}(\mbf{0}, \mbf{I}),
\end{align*}
where $\bfX \sim p$.
\end{customthm}
\begin{proof}
The proof closely parallels Theorem 10 in \cite{chen2001gaussianization} and Huber's proof of weak convergence for projection pursuit density estimates~\citep{huber1985projection}. Let $q^{(i)}$ denote the distribution of $\bfX^{(i)}$, and let $q^{(i,j)}$ denote the $j$-th marginal distribution of $\bfX^{i}$. Because $\Psi^{(i)}$ and $R^{(i)}$ are continuous operations for all $i$, and $p$ is supported on a compact set $\mcal{X}$, we know $q^{(i,j)}$ is also supported on a compact set, which we denote as $\mcal{X}^{(i,j)}$. In addition, we observe that the determinant of Jacobian of $\Psi^{(i)}$ is always strictly positive, and the determinant of $R^{(i)}$ is $\pm 1$. Combining this observation with the condition that $\inf_{\bfx\in\mcal{X}} p(\bfx) > \delta$, we know there exists a constant $\delta^{(i,j)} > 0$ such that $\inf_{x\in\mcal{X}^{(i,j)}} q^{(i,j)}(x) > \delta^{(i,j)}$ for all $i,j$. Therefore, $q^{(i,j)}$ and $\Psi^{(i)}$ satisfy the conditions in \lemref{lem:universal_kl} and \corref{cor:kl}, which entails that for any $\epsilon^{(i,j)} > 0$ and $\epsilon^{(i)} > 0$ there exists $\Psi^{(i)}$ such that $\forall j: \KLD{q^{(i,j)}(x)}{\mcal{N}(0,1)} < \epsilon^{(i,j)}$ and $J_m(\bfX^{(k)}) < \epsilon^{(i)}$.

The KL divergence of $\bfX^{(k)}$ is
\begin{align*}
    J(\bfX^{(k)}) &= I(\bfX^{(k)}) + J_m(\bfX^{(k)}) \\
    &= I(R^{(k)}\bfX^{(k-1)}) + J_m(\Psi^{(k)}R^{(k)}\bfX^{(k-1)}).
\end{align*}
After $q^{(k-1)}(\bfX^{(k-1)})$ is obtained, we always choose $R^{(k)} = \arginf_{R}I(R\bfX^{(k-1)})$, and then select $\Psi^{(k)}$ appropriately such that 
\begin{align*}
    J_m(\Psi^{(k)}R^{(k)}\bfX^{(k-1)}) \leq \epsilon^{(k)} = \frac{\epsilon^{(k-1)}}{2}
\end{align*}
Let $\Delta^{(k)}$ denote the reduction in the KL divergence in the $k$-th layer:
\begin{align*}
    \Delta^{(k)} &= J(\bfX^{(k)}) - J(\bfX^{(k+1)}) \\
    &= I(\bfX^{(k)}) + \epsilon^{(k)} - \inf_{R}I(R\bfX^{(k)}) - \epsilon^{(k+1)}\\
    &\geq 0.
\end{align*}
Since $\{J(\bfX^{(k)})\}$ is a monotonically decreasing sequence and bounded from below by $0$, we have
\begin{align*}
    \lim_{k \to \infty} \Delta^{(k)} = 0.
\end{align*}
Also since $\epsilon^{(k)} = \epsilon^{(k-1)} / 2$, we have $\lim_{k\to\infty} \epsilon^{(k)} = 0$.

Next, we consider the following quantity
\begin{align*}
    J^*(\bfX) = \max_{\norm{\bfalpha}_2 = 1}J(\bfalpha^\intercal \bfX).
\end{align*}
For any unit vector $\bfalpha$, we let $U_\bfalpha$ be an orthogonal completion of $\bfalpha$, \ie,
\begin{align*}
    U_\bfalpha = [\bfalpha, \bfalpha_2, \cdots, \bfalpha_n].
\end{align*}
Then, for any $\bfalpha$ with $\norm{\bfalpha}_2=1$, we have
\begin{align*}
    J(\bfalpha^\intercal \bfX^{(k)}) &\leq J_m(U_\bfalpha \bfX^{(k)})\\
    &= J(U_\bfalpha \bfX^{(k)}) - I(U_\bfalpha \bfX^{(k)})\\
    &\leq J(\bfX^{(k)}) - \inf_R I(R \bfX^{(k)})\\
    &= I(\bfX^{(k)}) + \epsilon^{(k)} - \inf_R I(R \bfX^{(k)})\\
    &= \Delta^{(k)} + \epsilon^{(k+1)}.
\end{align*}
Therefore,
\begin{align*}
J^*(\bfX^{(k)}) = \sup_{\norm{\bfalpha}_2=1} J(\bfalpha^\intercal \bfX^{(k)}) \leq \Delta^{(k)} + \epsilon^{(k+1)}.
\end{align*}
Since $\Delta^{(k)} \to 0$ and $\epsilon^{(k+1)} \to 0$, we know $J^*(\bfX^{(k)}) \to 0$. Applying Lemma 6 from \citet{chen2001gaussianization}, we get
\begin{align*}
    \bfX^{(k)} \stackrel{d}{\to} \mcal{N}(\mbf{0}, \mbf{I}),
\end{align*}
which proved our statement.
\end{proof}

\section{MORE DETAILS ON PATCH-BASED ROTATION MATRICES}\label{app:patch}
\begin{figure}[!ht]
    \centering
        \includegraphics[width=0.5\textwidth]{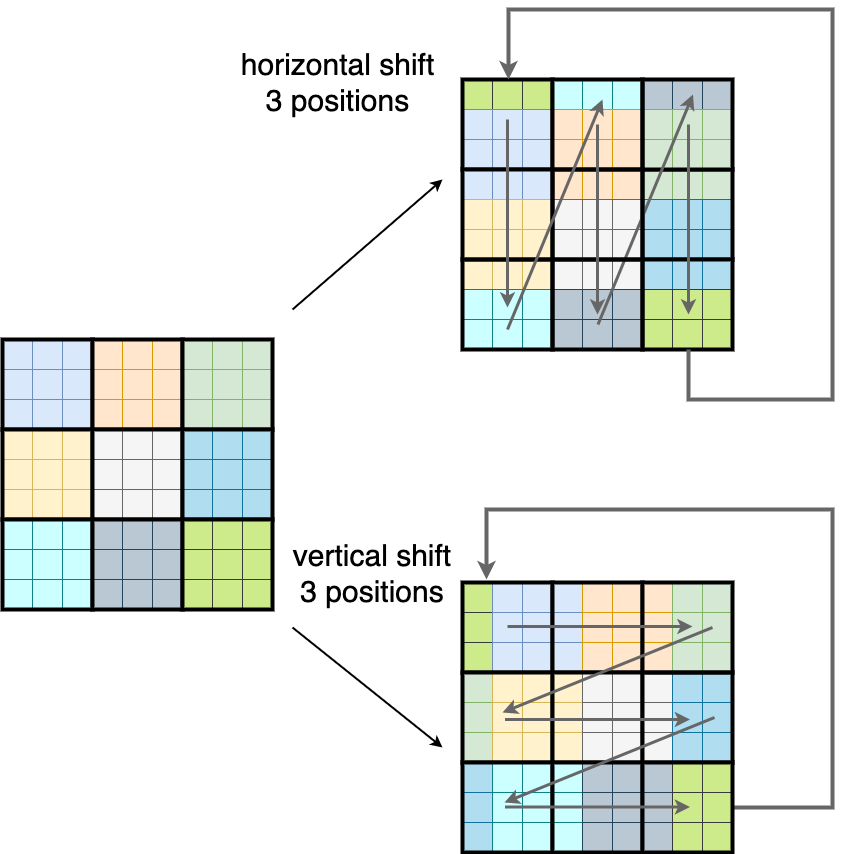}
    \caption{The shifting operation for constructing patch-based rotation matrices.}
    \label{fig:patch_rotation_appendix}
\end{figure}

The motivation for patch-based rotation is based on the following:
\begin{itemize}
    \item [1.] Naive Householder reflections are too expensive for high dimensional data.
    \item [2.] For image datasets, a pixel should be more correlated to its neighboring pixels than pixels that are far away.
\end{itemize}

Instead of parameterizing a rotation for the entire image, a patch-based rotation only focuses on the local image structure: it first partitions an image into non-overlapping groups of neighboring pixels, and then parametrizes a (smaller) rotation matrix for each group independently. To introduce dependency across different groups, we also shift the inputs at each layer. More specifically, we shift each pixel by $c$ positions horizontally (or vertically) in a circular way by wrapping around pixels that are discarded after the shift (see \figref{fig:patch_rotation_appendix}). 

\section{SAMPLES}
We provide uncurated samples from GF, MADE and RBIG for both MNIST and Fashion-MNIST datasets (see \figref{fig:mnist_samples} and \figref{fig:fmnist_samples}).%

\begin{figure}%
    \centering
    \begin{subfigure}[b]{0.3\textwidth}
        \includegraphics[width=\textwidth]{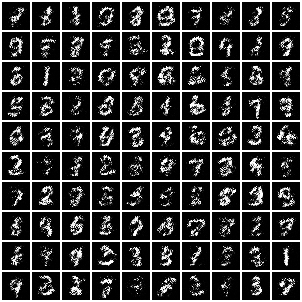}
    \end{subfigure}
    \begin{subfigure}[b]{0.3\textwidth}
        \includegraphics[width=\textwidth]{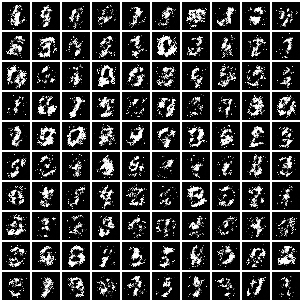}
    \end{subfigure}
    \begin{subfigure}[b]{0.3\textwidth}
        \includegraphics[width=\textwidth]{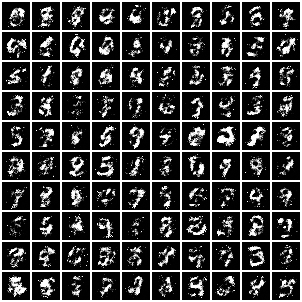}
    \end{subfigure}
    \caption{Uncurated MNIST samples from various models. \textbf{Left:} MADE. \textbf{Middle:} RBIG. \textbf{Right:} GF.}
    \label{fig:mnist_samples}
\end{figure}

\begin{figure}%
    \centering
    \begin{subfigure}[b]{0.3\textwidth}
        \includegraphics[width=\textwidth]{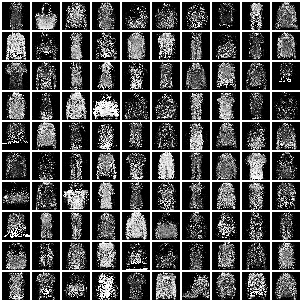}
    \end{subfigure}
    \begin{subfigure}[b]{0.3\textwidth}
        \includegraphics[width=\textwidth]{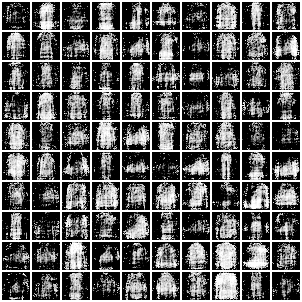}
    \end{subfigure}
    \begin{subfigure}[b]{0.3\textwidth}
        \includegraphics[width=\textwidth]{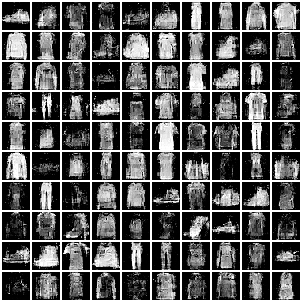}
    \end{subfigure}
    \caption{Uncurated Fashion-MNIST samples from various models. \textbf{Left:} MADE. \textbf{Middle:} RBIG. \textbf{Right:} GF.}
    \label{fig:fmnist_samples}
\end{figure}

\section{ADDITIONAL EXPERIMENTAL DETAILS FOR GF}
\label{sec:settings}

\begin{table*}
\caption{Model architectures and hyperparameters for Gaussianization Flows.}%
\label{tab:gf_architecture}
\begin{center}
    \begin{tabular}{p{3.0cm} |c c c c c c c}
        \toprule
        Method &POWER &GAS &HEPMASS &MINIBOONE &BSDS300 &MNIST &FMNIST\\
        \midrule
        layers &50 &150 &100  &90 &30 &10 &10\\
        anchor points &50 &50 &100 &50 &60 &50 &50\\
        Householder refl. &6 &8 &21 &43 &63 &50 &50\\
        patch size &- &- &- &- &- &4 &4\\
        learning rate &0.005 &0.005 &0.005 &0.005 &0.005 &0.01 &0.01\\
        epoch &200 &200 &200 & 200 &200 &200 &200 \\
        batch size &2000 &2000 &500 &500 &1000 &100 &100\\
        \bottomrule
    \end{tabular} 
\end{center}
\end{table*}

\subsection{Tabular and Image Datasets}
\label{sec:density_estimation}
We report the architectures and hyperparameters of our GF models in \tabref{tab:gf_architecture}. We tested the following three different parameterizations of the trainable rotation matrix layer.
\begin{itemize}
    \item [(1)] Na\"{i}ve Householder reflections.\\
    \item [(2)] Patch-based rotation matrices.\\
    \item [(3)] Alternating between (1) and (2).
\end{itemize}
We only use (1) for tabular datasets, where each rotation matrix is constructed using $D$ Householder reflections ($D$ denotes the dimensionality of datasets). We try all three parameterizations for MNIST and Fashion-MNIST. In the case of (1), we use 50 Householder reflections to trade off between computational efficiency and model expressivity. In the case of (2), we set patch size to 4 and randomly choose a different shifting constant $c$ for each layer. We observe that (2) achieves the best results on MNIST while (3) performs the best on Fashion-MNIST. We report the best results out of all parameterizations for MNIST and Fashion-MNIST in \tabref{tab:tabular_table}.

\subsection{Stretched Tabular Datasets}
We report the architectures and hyperparameters of our GF models in \tabref{tab:gf_architecture}. %
For other models, we adopt the default architectures and hyperparameters as mentioned in their original papers, except that we train each model for 100 epochs. 

\subsection{Small Training Sets}
We report the architectures and hyperparameters of our GF models in \tabref{tab:gf_architecture}. %
For other models, we adopt the default architectures and hyperparameters as mentioned in their original papers, except that we train each model for 200 epochs on a smaller shuffled subset.

\section{ADDITIONAL EXPERIMENTAL DETAILS FOR RBIG}
\label{sec:rbig_analysis}
For the marginal Gaussianization step, we fit the 1D distributions with KDE (logistic kernel) and tune the kernel bandwidth with the rule-of-thumb bandwidth estimator. For the rotation step, we use rotation matrices obtained from PCA rather than random, as we empirically observe that PCA exhibits better performance and faster convergence than random rotation matrices. 

To obtain better results, we perform KDE on 50000 data points for all the tabular datasets, except for BSDS300 where the number of samples is reduced to 20000 for faster computation. We use 10000 data points for both MNIST and Fashion-MNIST. The number of RBIG iterations is set to 5 for all datasets, as we found it performs the best within the range 1--100.

\end{document}